%% file: main.tex
\documentclass[letterpaper]{article}

\input{helpers}

\title{Two-Player Games for Efficient Non-Convex Constrained Optimization}
\author[1]{Andrew Cotter\thanks{acotter@google.com}}
\author[1]{Heinrich Jiang\thanks{heinrichj@google.com}}
\author[2]{Karthik Sridharan\thanks{sridharan@cs.cornell.edu}}
\affil[1]{Google AI}
\affil[2]{Cornell University}

\begin{document}

\maketitle

\begin{abstract}
\input{abstract}
\end{abstract}

\input{sec-introduction}
\input{sec-related-work}
\input{sec-lagrangian}
\input{sec-proxy-lagrangian}
\input{sec-overall}
\input{sec-experiments}

\newpage
\clearpage

\input{sec-acknowledgements}

\bibliography{main}
\bibliographystyle{plainnat}

\newpage
\clearpage
\onecolumn

\appendix

\showproofstrue

\input{app-proofs}
\end{document}

%% file: helpers.tex
\usepackage{amsfonts}  
\usepackage{amsmath}
\usepackage{amssymb}
\usepackage{amsthm}
\usepackage{appendix}
\usepackage{authblk}
\usepackage{bbold}
\usepackage{booktabs}  
\usepackage{color}
\usepackage{courier}
\usepackage{enumitem}
\usepackage{environ}
\usepackage{etoolbox}
\usepackage{float}
\usepackage{fullpage}
\usepackage{graphicx}
\usepackage{helvet}
\usepackage{mathtools}
\usepackage{mdwlist}  
\usepackage{microtype}  
\usepackage{multirow}
\usepackage{natbib}
\usepackage{nicefrac}  
\usepackage{parskip}
\usepackage{subfigure}
\usepackage[T1]{fontenc}  
\usepackage{tabularx}
\usepackage{times}
\usepackage{url}  
\usepackage[utf8]{inputenc}  
\usepackage{verbatim}
\usepackage{wrapfig}
\usepackage{xspace}

\usepackage{algorithm}
\usepackage{algorithmic}

\usepackage{hyperref}  

\setlist{nosep}

\hypersetup{colorlinks = false, pdfborder = {0 0 0}}

\newif\ifdraft
\draftfalse
\newcommand{\TODO}[1]{
  \ifdraft
    \textcolor{red}{[TODO: #1]}\xspace
  \fi
}
\newcommand{\NOTE}[1]{
  \ifdraft
    \textcolor{blue}{[NOTE: #1]}\xspace
  \fi
}

\renewcommand{\paragraph}[1]{\textbf{#1}\;\xspace}

\renewcommand{\eqref}[1]{Equation~\ref{eq:#1}}
\newcommand{\secref}[1]{Section~\ref{sec:#1}}
\newcommand{\appref}[1]{Appendix~\ref{app:#1}}
\newcommand{\figref}[1]{Figure~\ref{fig:#1}}
\newcommand{\tabref}[1]{Table~\ref{tab:#1}}
\newcommand{\algref}[1]{Algorithm~\ref{alg:#1}}
\newcommand{\thmref}[1]{Theorem~\ref{thm:#1}}
\newcommand{\lemref}[1]{Lemma~\ref{lem:#1}}
\newcommand{\corref}[1]{Corollary~\ref{cor:#1}}
\newcommand{\defref}[1]{Definition~\ref{def:#1}}

\newcommand{\thmrefs}[3][and]{Theorems~\ref{thm:#2} #1~\ref{thm:#3}}

\newtheoremstyle{mytheoremstyle}
  {\parsep} 
  {} 
  {\itshape} 
  {} 
  {\bfseries} 
  {.} 
  {0.5em} 
  {} 
\theoremstyle{mytheoremstyle}

\makeatletter
\newtheorem*{rep@theorem}{\rep@title}
\newcommand{\newreptheorem}[2]{\newenvironment{rep#1}[1]{\def\rep@title{#2 \ref{##1}}\begin{rep@theorem}}{\end{rep@theorem}}}
\makeatother

\newtheorem{theorem}{Theorem}
\newtheorem{lemma}{Lemma}
\newtheorem{corollary}{Corollary}
\newtheorem{definition}{Definition}
\newreptheorem{theorem}{Theorem}
\newreptheorem{lemma}{Lemma}
\newreptheorem{corollary}{Corollary}

\newif\ifreptheorem
\reptheoremfalse
\newif\ifshowproofs
\showproofsfalse

\newcommand{\replabel}[1]{
  \ifreptheorem
    \tag{\ref{#1}}
  \else
    \label{#1}
  \fi
}
\NewEnviron{thm}[1]{
  \makeatletter
  \ifcsname defined:thm:#1\endcsname
    \reptheoremtrue
    \begin{reptheorem}{thm:#1}\BODY\end{reptheorem}
    \reptheoremfalse
  \else
    \begin{theorem}\label{thm:#1}\BODY\end{theorem}
    \csgdef{defined:thm:#1}{}
  \fi
  \makeatother
}
\NewEnviron{lem}[1]{
  \makeatletter
  \ifcsname defined:lem:#1\endcsname
    \reptheoremtrue
    \begin{replemma}{lem:#1}\BODY\end{replemma}
    \reptheoremfalse
  \else
    \begin{lemma}\label{lem:#1}\BODY\end{lemma}
    \csgdef{defined:lem:#1}{}
  \fi
  \makeatother
}
\NewEnviron{cor}[1]{
  \makeatletter
  \ifcsname defined:cor:#1\endcsname
    \reptheoremtrue
    \begin{repcorollary}{cor:#1}\BODY\end{repcorollary}
    \reptheoremfalse
  \else
    \begin{corollary}\label{cor:#1}\BODY\end{corollary}
    \csgdef{defined:cor:#1}{}
  \fi
  \makeatother
}
\NewEnviron{prf}[1]{
  \ifshowproofs
    \begin{proof}\BODY\end{proof}\label{app:prf:#1}
  \else
    \begin{proof}In \appref{prf:#1}.\end{proof}
  \fi
}

\newenvironment{titled-paragraph}[1]{\textbf{#1:}}{}

\newcommand{\tensorflow}{TensorFlow\xspace}
\newcommand{\adagrad}{AdaGrad\xspace}

\newcommand{\wrt}{w.r.t.\xspace}
\newcommand{\ie}{i.e.\xspace}
\newcommand{\eg}{e.g.\xspace}
\newcommand{\egcite}[1]{\citep[\eg][]{#1}}

\newcommand{\probability}{\ensuremath{\mathrm{Pr}}}
\newcommand{\expectation}{\ensuremath{\mathbb{E}}}

\newcommand{\indicator}{\ensuremath{\mathbf{1}}}

\newcommand{\defeq}{\ensuremath{\vcentcolon=}}

\newcommand{\N}{\ensuremath{\mathbb{N}}}

\newcommand{\R}{\ensuremath{\mathbb{R}}}

\newcommand{\suchthat}[1][]{\ensuremath{\underset{#1}{\mathrm{s.t.\;}}}}

\newcommand{\inner}[2]{\ensuremath{\left\langle {#1}, {#2} \right\rangle}\xspace}

\newcommand{\norm}[1]{\ensuremath{\left\lVert {#1} \right\rVert}\xspace}

\newcommand{\abs}[1]{\ensuremath{\left\lvert {#1} \right\rvert}\xspace}

\DeclareMathOperator*{\argmin}{argmin}

\newcommand{\indices}[1]{\ensuremath{\left[#1\right]}}

\newcommand{\codecomment}[1]{\`//~\textit{#1}}
\newcounter{lineno}
\newenvironment{pseudocode}{\setcounter{lineno}{0}\begin{tabbing}\textbf{mm}\=mm\=mm\=mm\=mm\=\kill}{\end{tabbing}}
\newcommand{\codename}{\>}

\newcommand{\codeline}{\>\stepcounter{lineno}\textbf{\arabic{lineno}}\'\>}

\DeclareMathOperator{\fix}{fix}

\newcommand{\elementwiseproduct}{\odot}

\newcommand{\elementwiseexp}{\operatorname{.exp}}

\newcommand{\grad}{\ensuremath{\nabla}\xspace}
\newcommand{\subgrad}{\ensuremath{\check{\grad}}\xspace}
\newcommand{\supgrad}{\ensuremath{\hat{\grad}}\xspace}
\newcommand{\stochasticgrad}{\ensuremath{\Delta}\xspace}
\newcommand{\stochasticsubgrad}{\ensuremath{\check{\stochasticgrad}}\xspace}
\newcommand{\stochasticsupgrad}{\ensuremath{\hat{\stochasticgrad}}\xspace}

\newcommand{\parameters}{\ensuremath{\theta}\xspace}
\newcommand{\Parameters}{\ensuremath{\Theta}\xspace}
\newcommand{\multipliers}{\ensuremath{\lambda}\xspace}
\newcommand{\Multipliers}{\ensuremath{\Lambda}\xspace}
\newcommand{\matrixmultipliers}{\ensuremath{M}\xspace}
\newcommand{\Matrixmultipliers}{\ensuremath{\mathcal{M}}\xspace}

\newcommand{\dataset}{\ensuremath{S}\xspace}

\newcommand{\numconstraints}{\ensuremath{m}\xspace}
\newcommand{\matrixmultipliersize}{\ensuremath{\tilde{m}}\xspace}
\newcommand{\classifier}{\ensuremath{f}\xspace}
\newcommand{\constraint}[1]{\ensuremath{g_{#1}}\xspace}
\newcommand{\objective}{\constraint{0}\xspace}
\newcommand{\proxyconstraint}[1]{\ensuremath{\tilde{g}_{#1}}\xspace}
\newcommand{\lagrangian}{\ensuremath{\mathcal{L}}\xspace}

\newcommand{\bregman}[1]{\ensuremath{D_{#1}}\xspace}
\newcommand{\bound}[1]{\ensuremath{B_{#1}}\xspace}

\newcommand{\Radius}{\ensuremath{R}\xspace}
\newcommand{\margin}{\ensuremath{\gamma}\xspace}
\newcommand{\deltamatrix}{\ensuremath{A}\xspace}

\newcommand{\approximation}{\ensuremath{\rho}\xspace}
\newcommand{\oracle}{\ensuremath{\mathcal{O}_{\approximation}}\xspace}

%% file: abstract.tex
In recent years, constrained optimization has become increasingly relevant to
the machine learning community, with applications including Neyman-Pearson
classification, robust optimization, and fair machine learning. A natural
approach to constrained optimization is to optimize the Lagrangian, but this is
not guaranteed to work in the non-convex setting, and, if using a first-order
method, cannot cope with non-differentiable constraints (e.g. constraints on
rates or proportions).

The Lagrangian can be interpreted as a two-player game played between a player
who seeks to optimize over the model parameters, and a player who wishes to
maximize over the Lagrange multipliers.
We propose a non-zero-sum variant of the Lagrangian formulation that can cope
with non-differentiable---even discontinuous---constraints, which we call the
``proxy-Lagrangian''. The first player minimizes external regret in terms of
easy-to-optimize ``proxy constraints'', while the second player enforces the
\emph{original} constraints by minimizing swap regret.

For this new formulation, as for the Lagrangian in the non-convex setting, the
result is a stochastic classifier. For both the proxy-Lagrangian and Lagrangian
formulations, however, we prove that this classifier, instead of having
unbounded size, can be taken to be a distribution over no more than
$\numconstraints+1$ models (where $\numconstraints$ is the number of
constraints). This is a significant improvement in practical terms.

%% file: sec-introduction.tex
\section{Introduction}\label{sec:introduction}

We consider the general problem of inequality constrained optimization, in
which we wish to find a set of parameters $\parameters \in \Parameters$
minimizing an objective function subject to $\numconstraints$ functional
constraints:
\begin{align}
  \label{eq:constrained-problem} \min_{\parameters \in \Parameters} \; &
  \objective\left(\parameters\right) \\
  \notag \suchthat & \forall i \in \indices{\numconstraints} .
  \constraint{i}\left(\parameters\right) \le 0
\end{align}
%
%
To highlight some of the challenges that arise in non-convex constrained
optimization, consider the specific example of constraining a \emph{fairness}
metric. We cast the fairness problem as that of minimizing some empirical loss
subject to one or more fairness constraints. One of the simplest examples of
such is the following:
%
%
%
\begin{align}
  \label{eq:fairness-example} \min_{\parameters \in \Parameters} \; &
  \frac{1}{\abs{\dataset}} \sum_{x,y \in \dataset} \ell\left(
  \classifier\left(x; \parameters\right), y \right) \\
  \notag \suchthat & \frac{1}{\abs{\dataset}} \sum_{x \in
  \dataset_{\mathrm{min}}} \indicator_{\classifier\left(x; \parameters\right) >
  0} \ge \frac{0.8}{\abs{\dataset}} \sum_{x \in \dataset}
  \indicator_{\classifier\left(x; \parameters\right) > 0}
\end{align}
Here, $\classifier\left(\cdot;\parameters\right)$ is a classification function
with parameters $\parameters$, $\dataset$ is the training dataset, and
$\dataset_{\mathrm{min}} \subseteq \dataset$ represents a minority population.
The constraint represents a version of the so-called ``80\%
rule''~\egcite{Biddle:2005,Vuolo:2013}, and forces the resulting classifier to
make at least 80\% of its positive predictions on the minority
population---\citet{Goh:2016} and \citet{Narasimhan:2018} discuss a number of
useful constraints that are formulated similarly, both on fairness and
non-fairness metrics.
Unfortunately, several serious challenges arise when we attempt to optimize
this problem:
\begin{enumerate}
  \item \label{item:introduction:data-dependent} The constraint is
  data-dependent, and could therefore be very expensive to check.
  \item \label{item:introduction:non-convex} The classification function
  \classifier may be a badly-behaving function of \parameters (\eg a deep
  neural network), resulting in non-convex objective and constraint functions.
  \item \label{item:introduction:non-differentiable} Worse, the constraint is a
  linear combination of \emph{indicators}, hence is not even subdifferentiable
  \wrt \parameters.
\end{enumerate}

Perhaps the most ``familiar'' technique for constrained optimization is
to formulate the Lagrangian:
\begin{definition}
  \label{def:lagrangian}
  The Lagrangian $\lagrangian : \Parameters \times \Multipliers \rightarrow \R$
  of \eqref{constrained-problem} is:
  \begin{equation*}
    \lagrangian\left(\parameters, \multipliers\right) \defeq
    \objective\left(\parameters\right) + \sum_{i=1}^{\numconstraints}
    \multipliers_i \constraint{i}\left(\parameters\right)
  \end{equation*}
  where $\Multipliers \subseteq \R_+^{\numconstraints}$.
\end{definition}
and jointly minimize over $\parameters \in \Parameters$ and maximize over
$\multipliers \in \Multipliers \subseteq \R_+^{\numconstraints}$.
By itself, using this formulation doesn't address the challenges we identified
above, but we will see that, compared to the alternatives
(\secref{related-work:alternatives}), it's a good starting point for an
approach that does.

\subsection{Dealing with non-Convexity}\label{sec:introduction:non-convex}

\input{figures/fig-counterexample}
Optimizing the Lagrangian can be interpreted as playing a two player zero-sum
game: the first player chooses $\parameters$ to minimize
$\lagrangian\left(\parameters,\multipliers\right)$, and the second player
chooses $\multipliers$ to maximize it. The essential difficulty is that,
without strong duality---equivalently, unless the minimax theorem holds, giving
that $\min_{\parameters\in\Parameters} \max_{\multipliers\in\Multipliers}
\lagrangian\left(\parameters,\multipliers\right) =
\max_{\multipliers\in\Multipliers} \min_{\parameters\in\Parameters}
\lagrangian\left(\parameters,\multipliers\right)$---then the
$\parameters$-player, who is working on the primal (minimax) problem, and the
$\multipliers$-player, who is working on the dual (maximin) problem, might fail
to converge to a solution satisfying both players simultaneously (\ie a pure
Nash equilibrium).

If \eqref{constrained-problem} is a convex optimization problem and the action
spaces $\Parameters$ and $\Multipliers$ are compact and convex, then the
minimax theorem holds~\citep{Neumann:1928}, and optimizing the Lagrangian will
work. Otherwise it might not, and in fact it's quite easy to construct a
counterexample: \figref{counterexample} shows a case in which a pure Nash
equilibrium of the Lagrangian game \emph{does not exist}.
For this reason, the standard approach for handling non-convex machine learning
problems, \ie pretending that the problem is convex and using a stochastic
first order algorithm anyway,
%
%
should not be expected to reliably converge to a pure Nash equilibrium---even
on a problem as trivial as that in \figref{counterexample}---since there may be
none for it to converge \emph{to}.
%
%

Under general conditions, however, even when there is no \emph{pure} Nash
equilibrium, a \emph{mixed} equilibrium (\ie a pair of distributions over
$\parameters$ and $\multipliers$) does exist.
Such an equilibrium defines a stochastic classifier: upon receiving an example
$x$ to classify, we would sample $\parameters$ from its equilibrium
distribution, and then evaluate the classification function
$\classifier\left(x;\parameters\right)$.
Furthermore, and this is our first main contribution, this equilibrium can be
taken to consist of a discrete distribution over at most $\numconstraints+1$
distinct $\parameters$s ($\numconstraints$ being the number of constraints),
and a single non-random $\multipliers$.
This is a crucial improvement in practical terms, since a machine learning
model consisting of \eg a distribution over thousands (or more) of deep neural
networks---or worse, a continuous distribution---would likely be so unwieldy as
to be unusable.

\subsection{Introducing Proxy Constraints}\label{sec:introduction:non-zero-sum}

Most real-world machine learning implementations use first-order methods (even
on non-convex problems, \eg DNNs). To use such a method, however, one must have
gradients, and gradients are unavailable for non-differentiable constraints
like that in the fairness example of \eqref{fairness-example}, or in the myriad
of other situations in which one wishes to constrain \emph{counts} or
\emph{proportions} instead of smooth losses (\eg recall, coverage or churn as
in \citet{Goh:2016}). In all of these cases, the constraint functions are
piecewise-constant, so their gradients are zero almost everywhere, and a
gradient-based method cannot be expected to succeed.

The obvious solution is to use a surrogate. For example, we might consider
replacing the indicators of \eqref{fairness-example} with sigmoids, and then
optimizing the Lagrangian. This solves the differentiability problem, but
introduces a new one: a (mixed) Nash equilibrium would correspond to a solution
satisfying the sigmoid-relaxed constraint, instead of the \emph{actual}
constraint.
Interestingly, it turns out that we can seek to satisfy the original un-relaxed
constraint, even while using a surrogate. Our proposal is motivated by the
observation that, while differentiating the Lagrangian (\defref{lagrangian})
\wrt $\parameters$ requires differentiating the constraint functions
$\constraint{i}\left(\parameters\right)$, to differentiate it \wrt
$\multipliers$ we only need to \emph{evaluate} them. Hence, a surrogate is only
necessary for the $\parameters$-player; the $\multipliers$-player can continue
to use the original constraint functions.

We refer to a surrogate that is used by only one of the two players as a
``proxy'', and introduce the notion of ``proxy constraints'' by taking
$\proxyconstraint{i}\left(\parameters\right)$ to be a sufficiently-smooth upper
bound on $\constraint{i}\left(\parameters\right)$ for
$i\in\indices{\numconstraints}$, and formulating two functions that we call
``proxy-Lagrangians'':
\begin{definition}
  \label{def:proxy-lagrangians}
  Given proxy constraint functions $\proxyconstraint{i}\left(\parameters\right)
  \ge \constraint{i}\left(\parameters\right)$ for $i \in
  \indices{\numconstraints}$, the proxy-Lagrangians
  $\lagrangian_{\parameters},\lagrangian_{\multipliers} : \Parameters \times
  \Multipliers \rightarrow \R$ of \eqref{constrained-problem} are:
  \begin{align*}
    \lagrangian_{\parameters}\left(\parameters, \multipliers\right) \defeq&
    \multipliers_1 \objective\left(\parameters\right) +
    \sum_{i=1}^{\numconstraints} \multipliers_{i+1}
    \proxyconstraint{i}\left(\parameters\right) \\
    \lagrangian_{\multipliers}\left(\parameters, \multipliers\right) \defeq&
    \sum_{i=1}^{\numconstraints} \multipliers_{i+1}
    \constraint{i}\left(\parameters\right)
  \end{align*}
  where $\Multipliers \defeq \Delta^{\numconstraints+1} \ni \multipliers$ is the
  $\left(\numconstraints+1\right)$-dimensional simplex.
\end{definition}
%
%
As one might expect, the $\parameters$-player wishes to minimize
$\lagrangian_{\parameters}\left(\parameters, \multipliers\right)$, while the
$\multipliers$-player wishes to maximize
$\lagrangian_{\multipliers}\left(\parameters, \multipliers\right)$. Notice that
the $\proxyconstraint{i}$s are \emph{only} used by the $\parameters$-player.
Intuitively, the $\multipliers$-player chooses how much to weigh the proxy
constraint functions, but---and this is the key to our proposal---does so in
such a way as to satisfy the \emph{original} constraints.

Unfortunately, because the two players are optimizing different functions, this
is a non-zero-sum game, and finding a (mixed) Nash equilibrium of such games is
known to be PPAD-complete even in the finite setting~\citep{Chen:2006}. We
prove, however, that a \emph{weaker} type of equilibrium (a $\Phi$-correlated
equilibrium~\citep{Rakhlin:2011}, \ie a joint distribution over $\parameters$
and $\multipliers$ \wrt which neither player can improve)---one that we
\emph{can} find efficiently---suffices to guarantee a nearly-optimal and
nearly-feasible solution to \eqref{constrained-problem} in expectation.

\subsection{Contributions}\label{sec:introduction:contributions}

We first focus on the standard Lagrangian formulation, in the non-convex
setting. In \secref{lagrangian}, we provide an algorithm that, given access to
an approximate Bayesian optimization oracle, finds a stochastic classifier
that, in expectation, is provably approximately feasible and optimal. Many
previous authors have approached constrained optimization using similar
techniques (see \secref{related-work})---our main contribution is to show how
such a classifier can be efficiently ``shrunk'' to one that is \emph{at least
as good}, but is supported on only $\numconstraints+1$ solutions.

Our next major contribution is the introduction of the proxy-Lagrangian
formulation, which allows us to optimize constrained problems with extremely
general (even non-differentiable) constraints. In \secref{proxy-lagrangian}, we
prove that a particular type of $\Phi$-correlated equilibrium results in a
stochastic classifier that is feasible and optimal, and go on to provide a
novel algorithm that converges to such an equilibrium. Interestingly, to get
the ``right'' sort of equilibrium, the $\parameters$-player needs only minimize
the usual external regret, but the $\multipliers$-player must minimize the
\emph{swap regret}. While the resulting distribution is supported on a large
number of $\left(\parameters,\multipliers\right)$ pairs, applying the same
``shrinking'' procedure as before yields a distribution over only
$\numconstraints+1$ $\parameters$s that is at least as good as the original.

Finally, in \secref{overall}, we tie everything together by describing an
end-to-end recipe for provably solving a non-convex constrained optimization
problem with potentially non-differentiable constraints, yielding a stochastic
model that is a supported on at most $\numconstraints + 1$ solutions. In
practice, one would use SGD instead of an oracle, which results in an efficient
procedure that can be easily plugged-in to existing workflows, as is
experimentally verified in \secref{experiments}.

%% file: figures/fig-counterexample.tex
\begin{wrapfigure}{R}{0.5\textwidth}

\vspace{-4em}

\centering

\includegraphics[width=0.45\textwidth]{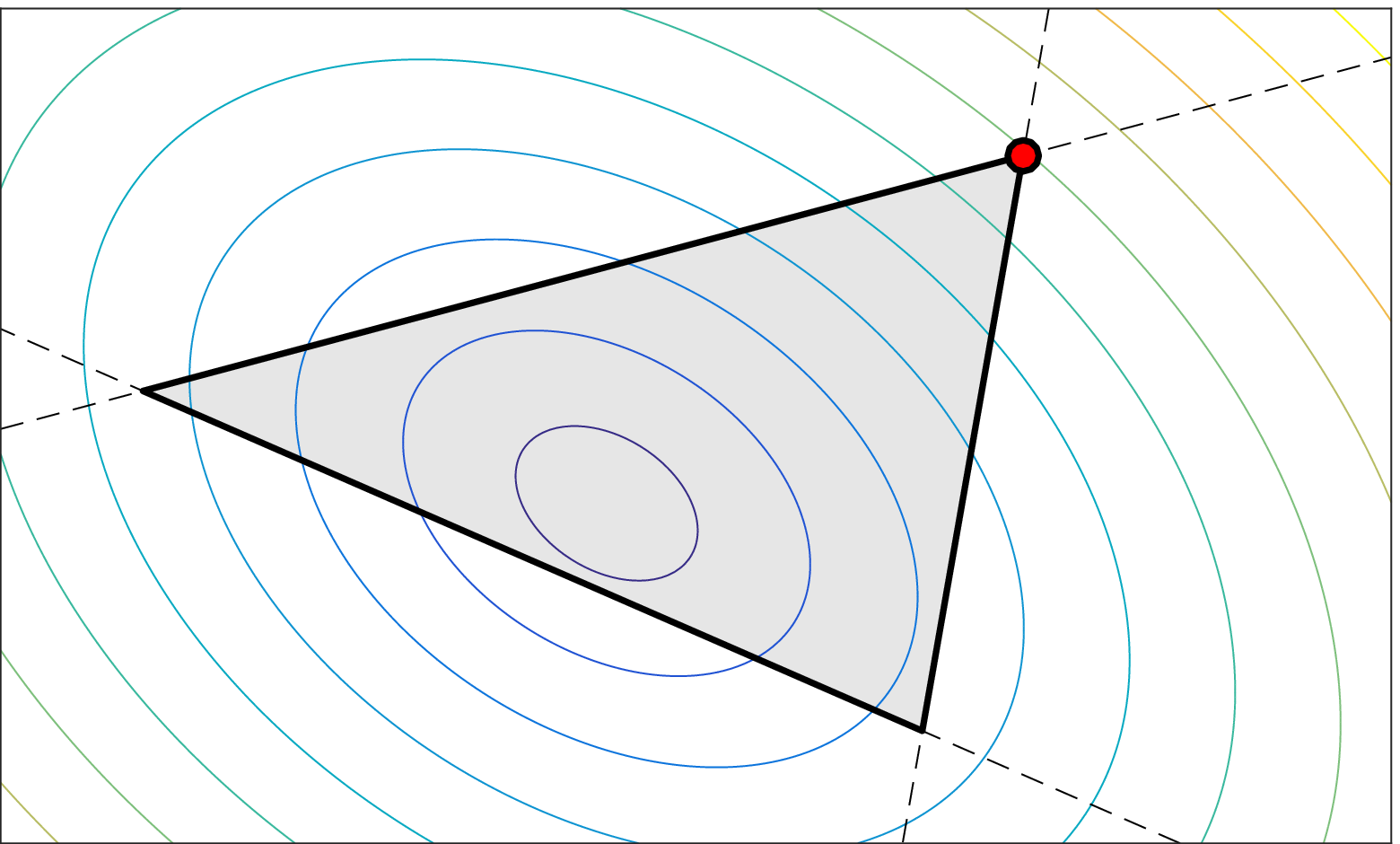}

\caption{
  %
  %
  The plotted rectangular region is the domain $\Parameters$, the contours are
  those of the \emph{strictly concave} minimization objective function
  $\objective$, and the shaded triangle is the feasible region determined by
  the three linear inequality constraints $\constraint{1}, \dots,
  \constraint{3}$. The red dot is the optimal feasible point.
  The Lagrangian $\lagrangian\left(\parameters, \multipliers\right)$ is
  strictly concave in $\parameters$ for any choice of $\multipliers$, so the
  optimal choice(s) for the $\parameters$-player will always lie on the four
  corners of the plotted rectangle. However, these points are infeasible, and
  therefore suboptimal for the $\multipliers$-player.
  %
  %
  %
}

\label{fig:counterexample}

\vspace{-1em}

\end{wrapfigure}

%% file: sec-related-work.tex
\section{Related Work}\label{sec:related-work}

The interpretation of constrained optimization as a two-player game has a long
history: \citet{Arora:2012} surveys some such work, and there are several
more recent examples~\egcite{Kearns:2017,Narasimhan:2018,Agarwal:2018}.
In particular, \citet{Agarwal:2018} propose an algorithm for fair
classification that is very similar to the Lagrangian-based approach that we
outline in \secref{lagrangian}---the main differences are our introduction of
``shrinking'', and that our setting (\eqref{constrained-problem}) is more
general.
The recent work of \citet{Chen:2017} addresses non-convex robust optimization,
\ie problems of the form:
\begin{equation*}
  \min_{\parameters \in \Parameters} \; \max_{i \in \indices{\numconstraints}}
  \constraint{i}\left(\parameters\right)
\end{equation*}
Like both us and \citet{Agarwal:2018}, they: (i) model such a problem as a
two-player game where one player chooses a mixture of objective functions, and
the other player minimizes the loss of the mixture, and (ii) they find a
\emph{distribution} over solutions rather than a pure equilibrium. These
similarities are unsurprising in light of the fact that robust optimization can
be reformulated as constrained optimization via the introduction of a slack
variable:
\begin{align}
  \label{eq:related-work:robust}
  \min_{\parameters \in \Parameters, \xi \in \Xi} \; & \xi \\
  \notag \suchthat & \forall i \in \indices{\numconstraints} . \xi \ge
  \constraint{i}\left(\parameters\right)
\end{align}
%
%
Correspondingly, one can transform a robust problem to a constrained one at the
cost of an extra bisection search~\egcite{Christiano:2011,Rakhlin:2013}. As
this relationship suggests, our main contributions can be adapted to the robust
optimization setting. In particular: (i) our proposed shrinking procedure can be
applied to \eqref{related-work:robust} to yield a distribution over only $m+1$
solutions, and (ii) one could perform robust optimization over
non-differentiable (even discontinuous) losses using ``proxy objectives'', just
as we use proxy constraints.


\subsection{Alternative Approaches}\label{sec:related-work:alternatives}

Given the difficulties involved in using a Lagrangian-like formulation for
non-convex problems, it's natural to ask whether one should instead favor a
procedure based on entirely different principles. Unfortunately, the potential
alternatives each present their own challenges.

The potential complexity of the constraints all but rules out approaches based
on projections (\eg projected SGD) or optimization of constrained
subproblems~(\eg Frank-Wolfe, as in \citet{Hazan:2012,Jaggi:2013,Garber:2013}).
Similarly, attempting to penalize
violations~\egcite{Arora:2012,Rakhlin:2013,Mahdavi:2012,Cotter:2016}, for
example by adding $\gamma \max_{i \in \indices{m}} \max\left\{0,
\constraint{i}\left(\parameters\right)\right\}$ to the objective, where
$\gamma\in\R_+$ is a hyperparameter, and optimizing the resulting problem using
a first order method, fails if the constraint functions are non-differentiable.
Even if they are, they may still be data-dependent, so evaluating
$\constraint{i}$, or even determining whether it is positive (as is necessary
for such techniques, due to the max with $0$), requires enumerating over the
entire dataset.  Hence, unlike the Lagrangian and proxy-Lagrangian
formulations, such ``penalized'' formulations are incompatible with the use of
a computationally-cheap stochastic optimizer.

In response to the idea of proxy constraints, it's natural to ask ``why not
just relax the constraints for \emph{both} players, instead of just the
$\parameters$-player?''.
This is indeed a popular approach, having been proposed \eg for Neyman-Pearson
classification~\citep{Davenport:2010,Bottou:2011}, more general rate
metrics~\citep{Goh:2016}, and AUC~\citep{Eban:2017}.
The answer is that in many cases, particularly when constraints are data
dependent, they represent real-world restrictions on how the learned model is
permitted to behave. For example, the ``80\% rule'' of \eqref{fairness-example}
can be found in the HOPA Act of 1995~\citep{Wiki:HOPA:2018}, and it requires an
80\% threshold in terms of the \emph{number of positive predictions}---not a
relaxation---which is precisely the target that the proxy-Lagrangian approach
will attempt to hit.
%

This point, in turn, raises the question of \emph{generalization}: satisfying
the correct un-relaxed constraints on training data does not necessarily mean
that they will be satisfied at evaluation time. This issue is outside the scope
of this paper, but is vital.
%
%
%
For certain specific applications, the post-training correction approach of
\citet{Woodworth:2017} can improve generalization performance, and
\citet{Cotter:2018}'s more recent proposal (which is based on our
proxy-Lagrangian formulation) can be applied more generally, but there is still
room for future work.

%% file: sec-lagrangian.tex
\section{Starting Point: Lagrangian Optimization}\label{sec:lagrangian}

\input{figures/alg-oracle-lagrangian}
Our ultimate interest is in constrained optimization, so before we present our
proposed algorithm for optimizing the Lagrangian (\defref{lagrangian}) in the
non-convex setting, we will characterize the relationship between an
approximate Nash equilibrium of the Lagrangian game, and a nearly-optimal
nearly-feasible solution to the original constrained problem
(\eqref{constrained-problem}):
\begin{theorem}
  \label{thm:lagrangian-suboptimality-and-feasibility}
  Define $\Multipliers \defeq \left\{ \multipliers \in \R_+^{\numconstraints} :
  \norm{\multipliers}_1 \le \Radius \right\}$, and let
  $\parameters^{(1)},\dots,\parameters^{(T)} \in \Parameters$ and
  $\multipliers^{(1)},\dots,\multipliers^{(T)} \in \Multipliers$ be sequences of
  parameter vectors and Lagrange multipliers
  %
  %
  that comprise an approximate mixed Nash equilibrium, \ie:
  \begin{equation*}
    \max_{\multipliers^* \in \Multipliers} \frac{1}{T} \sum_{t=1}^T
    \lagrangian\left( \parameters^{(t)}, \multipliers^* \right) -
    \inf_{\parameters^* \in \Parameters} \frac{1}{T} \sum_{t=1}^T
    \lagrangian\left( \parameters^*, \multipliers^{(t)} \right) \le \epsilon
  \end{equation*}
  Define $\bar{\parameters}$ as a random variable for which $\bar{\parameters}
  = \parameters^{(t)}$ with probability $1/T$, and let $\bar{\multipliers}
  \defeq \left(\sum_{t=1}^T \multipliers^{(t)}\right) / T$.
  Then $\bar{\parameters}$ is nearly-optimal in expectation:
  \begin{equation*}
    \expectation_{\bar{\parameters}}\left[
    \objective\left(\bar{\parameters}\right) \right] \le \inf_{\parameters^*
    \in \Parameters : \forall i .  \constraint{i}\left(\parameters^*\right) \le
    0} \objective\left(\parameters^*\right) + \epsilon
  \end{equation*}
  and nearly-feasible:
  \begin{equation}
    \label{eq:thm:lagrangian-suboptimality-and-feasibility:feasibility}
    \max_{i \in \indices{\numconstraints}}
    \expectation_{\bar{\parameters}}\left[
    \constraint{i}\left(\bar{\parameters}\right) \right] \le
    \frac{\epsilon}{\Radius - \norm{\bar{\multipliers}}_1}
  \end{equation}
  Additionally, if there exists a $\parameters' \in \Parameters$ that satisfies
  all of the constraints with margin $\gamma$ (\ie
  $\constraint{i}\left(\parameters'\right) \le -\gamma$ for all
  $i\in\indices{\numconstraints}$), then:
  \begin{equation*}
    \norm{ \bar{\multipliers} }_1 \le \frac{\epsilon +
    \bound{\objective}}{\gamma}
  \end{equation*}
  where $\bound{\objective} \ge \sup_{\parameters \in \Parameters}
  \objective\left(\parameters\right) - \inf_{\parameters \in \Parameters}
  \objective\left(\parameters\right)$ is a bound on the range of the objective
  function $\objective$.
\end{theorem}
\begin{proof}
  This is a special case of \thmref{lagrangian-suboptimality} and
  \lemref{lagrangian-feasibility} in \appref{suboptimality}.
\end{proof}
This theorem has a few differences from the more typically-encountered
equivalence between Nash equilibria and optimal feasible solutions in the
convex setting. First, it characterizes \emph{mixed} equilibria, in that
uniformly sampling from the sequences $\parameters^{(t)}$ and
$\multipliers^{(t)}$ can be interpreted as defining distributions over
$\Parameters$ and $\Multipliers$. A convexity assumption would enable us to
eliminate this added complexity by appealing to Jensen's inequality to replace
these sequences with their averages. Second, for the technical reason that we
require compact domains in order to prove convergence rates (below),
$\Multipliers$ is taken to consist only of sets of Lagrange multipliers with
bounded $1$-norm~\footnote{In \appref{suboptimality}, this is generalized to
$p$-norms.}.

Finally, as a consequence of this second point, the feasibility guarantee of
\eqref{thm:lagrangian-suboptimality-and-feasibility:feasibility} only holds if
the Lagrange multipliers are, on average, smaller than the maximum $1$-norm
radius $\Radius$. Thankfully, as is shown by the final result of
\thmref{lagrangian-suboptimality-and-feasibility}, if there exists a point
satisfying the constraints with some margin $\gamma$, then there will exist
$R$s that are large enough to guarantee feasibility to within $O(\epsilon)$.

Our proposed algorithm (\algref{oracle-lagrangian}) requires an \emph{oracle}
that performs approximate non-convex minimization, similarly to
\citet{Chen:2017}'s algorithm for robust optimization and
\citet{Agarwal:2018}'s for fair classification (the latter reference uses the
terminology ``best response''):
\begin{definition}
  \label{def:oracle}
  A $\approximation$-approximate Bayesian optimization oracle is a function
  $\oracle : \left(\Parameters \rightarrow \R\right) \rightarrow \Parameters$
  for which:
  \begin{equation*}
    f\left( \oracle\left(f\right) \right) \le \inf_{\parameters^* \in
    \Parameters} f\left(\parameters^*\right) + \approximation
  \end{equation*}
  for any $f : \Parameters \rightarrow \R$ that can be written as a nonnegative
    linear combination of the objective and constraint functions
    $\objective,\constraint{1},\dots,\constraint{\numconstraints}$.
\end{definition}
The $\parameters$-player uses this oracle, and the $\multipliers$-player uses
projected gradient ascent. Notice that, unlike the oracle of \citet{Chen:2017},
which provides a multiplicative approximation, $\oracle$ provides an
\emph{additive} approximation. \algref{oracle-lagrangian}'s convergence rate
is:
\input{theorems/lem-oracle-lagrangian}
Combined with \thmref{lagrangian-suboptimality-and-feasibility}, we therefore
have that if $\Radius$ is sufficiently large, then \algref{oracle-lagrangian}
will converge to a distribution over $\Parameters$ that is, in expectation,
$O(\approximation)$-far from being optimal and feasible at a $O(1/\sqrt{T})$
rate, where $\approximation$ is as in \defref{oracle}.

\subsection{Shrinking}\label{sec:lagrangian:shrinking}

\NOTE{we could move the lemma (but not the rest) to an appendix, if needed}
Aside from the unrealistic oracle assumption (which will be partially addressed
in \secref{proxy-lagrangian}), the main disadvantage of
\algref{oracle-lagrangian} is that it results in a mixture of $T$ models, which
presumably would be far too many to use in practice. However, much smaller Nash
equilibria exist:
\begin{lemma}
  \label{lem:sparse-lagrangian}
  If $\Parameters$ is a compact Hausdorff space, $\Multipliers$ is compact, and
  the objective and constraint functions
  $\objective,\constraint{1},\dots,\constraint{\numconstraints}$ are
  continuous, then the Lagrangian game (\defref{lagrangian}) has a mixed Nash
  equilibrium pair $\left(\parameters,\multipliers\right)$ where $\parameters$
  is a random variable supported on at most $\numconstraints+1$ elements of
  $\Parameters$, and $\multipliers$ is non-random.
\end{lemma}
\begin{proof}
  Follows from \thmref{sparse-equilibrium} in \appref{sparsity}.
\end{proof}
Of course, the mere existence of such an equilibrium is insufficient---we need
to be able to \emph{find} it, and \algref{oracle-lagrangian} manifestly does
not. Thankfully, we can re-formulate the problem of finding the optimal
$\epsilon$-feasible mixture of the $\parameters^{(t)}$s as a linear program
(LP) that can be solved to ``shrink'' the support set. We must first evaluate
the objective and constraint functions for every $\parameters^{(t)}$, yielding
a $T$-dimensional vector of objective function values, and $\numconstraints$
such vectors of constraint function evaluations, which are then used to specify
the LP:
\input{theorems/lem-sparse-linear-program}
This result suggests a two-phase approach to optimization. In the first phase,
we apply \algref{oracle-lagrangian}, yielding a sequence of iterates for which
the uniform distribution over the $\parameters^{(t)}$s is approximately
feasible and optimal. We then apply the procedure of
\lemref{sparse-linear-program} to find the \emph{best} distribution over these
iterates, which in particular is guaranteed to be no worse than the uniform
distribution, and is supported on at most $\numconstraints+1$ iterates. We'll
expand upon this further in \secref{overall}.

%% file: figures/alg-oracle-lagrangian.tex
\begin{algorithm*}[t]

\begin{pseudocode}
\codename $\mbox{OracleLagrangian}\left( \Radius \in \R_+, \lagrangian : \Parameters \times \Multipliers \rightarrow \R, \oracle : \left(\Parameters \rightarrow \R\right) \rightarrow \Parameters, T \in \N, \eta_{\multipliers} \in \R_+ \right)$: \\
\codeline Initialize $\multipliers^{(1)} = 0$ \\
\codeline For $t \in \indices{T}$: \\
\codeline \> Let $\parameters^{(t)} = \oracle\left( \lagrangian\left(\cdot,\multipliers^{(t)}\right) \right)$ \codecomment{Oracle optimization} \\
\codeline \> Let $\stochasticgrad^{(t)}_{\multipliers}$ be a gradient of $\lagrangian\left(\parameters^{(t)},\multipliers^{(t)}\right)$ \wrt $\multipliers$ \\
\codeline \> Update $\multipliers^{(t+1)} = \Pi_{\Multipliers}\left( \multipliers^{(t)} + \eta_{\multipliers} \stochasticgrad^{(t)}_{\multipliers} \right)$ \codecomment{Projected gradient update} \\
\codeline Return $\parameters^{(1)},\dots,\parameters^{(T)}$ and $\multipliers^{(1)},\dots,\multipliers^{(T)}$
\end{pseudocode}

\caption{
  Optimizes the Lagrangian formulation (\defref{lagrangian}) in the non-convex
  setting via the use of an approximate Bayesian optimization oracle $\oracle$
  (\defref{oracle}) for the $\parameters$-player.
  The parameter $\Radius$ is the radius of the Lagrange multiplier space
  $\Multipliers \defeq \left\{ \multipliers \in \R_+^{\numconstraints} :
  \norm{\multipliers}_1 \le \Radius \right\}$, and the function
  $\Pi_{\Multipliers}$ projects its argument onto $\Multipliers$ \wrt the
  Euclidean norm.
}

\label{alg:oracle-lagrangian}

\end{algorithm*}

%% file: theorems/lem-oracle-lagrangian.tex
\begin{lem}{oracle-lagrangian}
  \ifshowproofs
  \textbf{(\algref{oracle-lagrangian})}
  \fi
  Suppose that $\Multipliers$ and $\Radius$ are as in
  \thmref{lagrangian-suboptimality-and-feasibility}, and define the upper bound
  $\bound{\stochasticgrad} \ge \max_{t \in \indices{T}}
  \norm{\stochasticgrad_{\multipliers}^{(t)}}_2$.

  If we run \algref{oracle-lagrangian} with the step size $\eta_{\multipliers}
  \defeq \Radius / \bound{\stochasticgrad} \sqrt{2T}$, then the result
  satisfies the conditions of \thmref{lagrangian-suboptimality-and-feasibility}
  for:
  \begin{equation*}
    \epsilon = \approximation + \Radius \bound{\stochasticgrad}
    \sqrt{\frac{2}{T}}
  \end{equation*}
  where $\approximation$ is the error associated with the oracle $\oracle$.
\end{lem}
\begin{prf}{oracle-lagrangian}
  Applying \corref{sgd} to the optimization over $\multipliers$ gives:
  \begin{equation*}
    \frac{1}{T} \sum_{t=1}^T \lagrangian\left( \parameters^{(t)},
    \multipliers^* \right) - \frac{1}{T} \sum_{t=1}^T \lagrangian\left(
    \parameters^{(t)}, \multipliers^{(t)} \right) \le \bound{\Multipliers}
    \bound{\stochasticgrad} \sqrt{\frac{2}{T}}
  \end{equation*}
  By the definition of $\oracle$ (\defref{oracle}):
  \begin{equation*}
    \frac{1}{T} \sum_{t=1}^T \lagrangian\left( \parameters^{(t)},
    \multipliers^* \right) - \inf_{\parameters^* \in \Parameters} \frac{1}{T}
    \sum_{t=1}^T \lagrangian\left( \parameters^*, \multipliers^{(t)} \right)
    \le \approximation + \bound{\Multipliers} \bound{\stochasticgrad}
    \sqrt{\frac{2}{T}}
  \end{equation*}
  Using the linearity of $\lagrangian$ in $\multipliers$, the fact that
  $\bound{\Multipliers} = \Radius$, and the definitions of $\bar{\parameters}$
  and $\bar{\multipliers}$, yields the claimed result.
\end{prf}

%% file: theorems/lem-sparse-linear-program.tex
\begin{lem}{sparse-linear-program}
  Let $\parameters^{(1)},\parameters^{(2)},\dots,\parameters^{(T)} \in
  \Parameters$ be a sequence of $T$ ``candidate solutions'' of
  \eqref{constrained-problem}.
  Define $\vec{\objective},\vec{\constraint{i}} \in \R^T$ such that
  $\left(\vec{\objective}\right)_t = \objective\left(\parameters^{(t)}\right)$
  and $\left(\vec{\constraint{i}}\right)_t =
  \constraint{i}\left(\parameters^{(t)}\right)$ for
  $i\in\indices{\numconstraints}$, and consider the linear program:
  \begin{align*}
    \min_{p \in \Delta^T} \; & \inner{p}{\vec{\objective}} \\
    \suchthat & \forall i \in \indices{\numconstraints} .
    \inner{p}{\vec{\constraint{i}}} \le \epsilon
  \end{align*}
  where $\Delta^T$ is the $T$-dimensional simplex. Then every vertex $p^*$ of
  the feasible region---in particular an optimal one---has at most
  $\numconstraints^* + 1 \le \numconstraints + 1$ nonzero elements, where
  $\numconstraints^*$ is the number of active
  $\inner{p^*}{\vec{\constraint{i}}} \le \epsilon$ constraints.
\end{lem}
\begin{prf}{sparse-linear-program}
  The linear program contains not only the $\numconstraints$ explicit
  linearized functional constraints, but also, since $p \in \Delta^T$, the $T$
  nonnegativity constraints $p_t \ge 0$, and the sum-to-one constraint
  $\sum_{t=1}^T p_t = 1$.

  Since $p$ is $T$-dimensional, every vertex $p^*$ of the feasible region must
  include $T$ active constraints. Letting $\numconstraints^* \le
  \numconstraints$ be the number of active linearized functional constraints,
  and accounting for the sum-to-one constraint, it follows that at least $T -
  \numconstraints^* - 1$ nonnegativity constraints are active, implying that
  $p^*$ contains at most $\numconstraints^* + 1$ nonzero elements.
\end{prf}

%% file: sec-proxy-lagrangian.tex
\section{Proxy-Lagrangian Optimization}\label{sec:proxy-lagrangian}

\input{figures/alg-stochastic-proxy-lagrangian}
While the Lagrangian formulation can be used to solve constrained problems in
the form of \eqref{constrained-problem}, \algref{oracle-lagrangian} isn't
actually implementable, due to its reliance on an oracle. If one wished to
apply it in practice, one would need to replace the oracle with something else,
and for large-scale machine learning problems, ``something else'' is
overwhelmingly likely to be SGD~\citep{Robbins:1951,Zinkevich:2003} or another
first-order stochastic algorithm (\eg \adagrad~\citep{Duchi:2010} or
ADAM~\citep{Kingma:2015}).

This leads to the issue we raised in \secref{introduction:non-zero-sum}: for
non-differentiable constraints like those in the fairness example of
\eqref{fairness-example}, we cannot compute gradients, and therefore cannot use
a first-order algorithm. ``Fixing'' this issue by replacing the constraints
with differentiable surrogates introduces a new difficulty: solutions to the
resulting problem will satisfy the \emph{surrogate} constraints, rather than
the \emph{actual} constraints.

The proxy-Lagrangian formulation of \defref{proxy-lagrangians} sidesteps this
issue by using a non-zero-sum two-player game. The $\multipliers$-player
chooses how much the $\parameters$-player should penalize the (differentiable)
proxy constraints, but does so in such a way as to satisfy the \emph{original}
constraints.
Unfortunately, since the proxy-Lagrangian game is non-zero-sum, we cannot
expect to find a Nash equilibrium, at least not efficiently. However, the
analogous result to \thmref{lagrangian-suboptimality-and-feasibility} requires
a \emph{weaker} type of equilibrium: a joint distribution over $\Theta$ and
$\Lambda$ \wrt which the $\theta$-player can only make a negligible improvement
compared to the best constant strategy, and the $\lambda$-player compared to
the best action-swapping strategy (this is a particular type of
$\Phi$-correlated equilibrium~\citep{Rakhlin:2011}):
\begin{theorem}
  \label{thm:proxy-lagrangian-suboptimality-and-feasibility}
  Define
  %
  %
  $\Matrixmultipliers$
  as the set of all left-stochastic $\left(\numconstraints + 1\right) \times
  \left(\numconstraints + 1\right)$ matrices, $\Multipliers \defeq
  \Delta^{\numconstraints+1}$ as the
  $\left(\numconstraints+1\right)$-dimensional simplex, and assume that each
  $\proxyconstraint{i}$ upper bounds the corresponding $\constraint{i}$.
  Let $\parameters^{(1)},\dots,\parameters^{(T)} \in \Parameters$ and
  $\multipliers^{(1)},\dots,\multipliers^{(T)} \in \Multipliers$ be sequences
  %
  %
  satisfying:
  \begin{align*}
    \frac{1}{T} \sum_{t=1}^T \lagrangian_{\parameters}\left( \parameters^{(t)},
    \multipliers^{(t)} \right) - \inf_{\parameters^* \in \Parameters}
    \frac{1}{T} \sum_{t=1}^T \lagrangian_{\parameters}\left( \parameters^*,
    \multipliers^{(t)} \right) \le& \epsilon_{\parameters} \\
    \max_{\matrixmultipliers^* \in \Matrixmultipliers} \frac{1}{T} \sum_{t=1}^T
    \lagrangian_{\multipliers}\left( \parameters^{(t)}, \matrixmultipliers^*
    \multipliers^{(t)} \right) - \frac{1}{T} \sum_{t=1}^T
    \lagrangian_{\multipliers}\left( \parameters^{(t)}, \multipliers^{(t)}
    \right) \le& \epsilon_{\multipliers}
  \end{align*}
  Define $\bar{\parameters}$ as a random variable for which $\bar{\parameters}
  = \parameters^{(t)}$ with probability $\multipliers^{(t)}_1 / \sum_{s=1}^T
  \multipliers^{(s)}_1$, and let $\bar{\multipliers} \defeq \left(\sum_{t=1}^T
  \multipliers^{(t)}\right) / T$.
  Then $\bar{\parameters}$ is nearly-optimal in expectation:
  \begin{equation}
    \label{eq:thm:proxy-lagrangian-suboptimality-and-feasibility:optimality}
    \expectation_{\bar{\parameters}}\left[
    \objective\left(\bar{\parameters}\right) \right] \le \inf_{\parameters^*
    \in \Parameters : \forall i .
    \proxyconstraint{i}\left(\parameters^*\right) \le 0} \objective\left(
    \parameters^* \right) + \frac{\epsilon_{\parameters} +
    \epsilon_{\multipliers}}{\bar{\multipliers}_1}
  \end{equation}
  and nearly-feasible:
  \begin{equation}
    \label{eq:thm:proxy-lagrangian-suboptimality-and-feasibility:feasibility}
    \max_{i \in \indices{\numconstraints}}
    \expectation_{\bar{\parameters}}\left[
    \constraint{i}\left(\bar{\parameters}\right) \right] \le
    \frac{\epsilon_{\multipliers}}{\bar{\multipliers}_1}
  \end{equation}
  Additionally, if there exists a $\parameters' \in \Parameters$ that satisfies
  all of the proxy constraints with margin $\gamma$ (\ie
  $\proxyconstraint{i}\left(\parameters'\right) \le -\gamma$ for all
  $i\in\indices{\numconstraints}$), then:
  \begin{equation*}
    \bar{\multipliers}_1 \ge \frac{\gamma - \epsilon_{\parameters} -
    \epsilon_{\multipliers}}{\gamma + \bound{\objective}}
  \end{equation*}
  where $\bound{\objective} \ge \sup_{\parameters \in \Parameters}
  \objective\left(\parameters\right) - \inf_{\parameters \in \Parameters}
  \objective\left(\parameters\right)$ is a bound on the range of the objective
  function $\objective$.
\end{theorem}
\begin{proof}
  This is a special case of \thmref{proxy-lagrangian-suboptimality} and
  \lemref{proxy-lagrangian-feasibility} in \appref{suboptimality}.
\end{proof}
Notice that while
\eqref{thm:proxy-lagrangian-suboptimality-and-feasibility:feasibility}
guarantees feasibility \wrt the original constraints, the comparator in
\eqref{thm:proxy-lagrangian-suboptimality-and-feasibility:optimality} is
feasible \wrt the \emph{proxy} constraints. Hence, the overall guarantee is no
better than what we would achieve if we took $\constraint{i} \defeq
\proxyconstraint{i}$ for all $i \in \indices{\numconstraints}$, and optimized
the Lagrangian as in \secref{lagrangian}.
However, as will be demonstrated experimentally in \secref{experiments:adult},
because the feasible region \wrt the original constraints is larger (perhaps
significantly so) than that \wrt the proxy constraints, the proxy-Lagrangian
approach has more ``room'' to find a better solution in practice.

One key difference between this result and
\thmref{lagrangian-suboptimality-and-feasibility} is that the $\Radius$
parameter is absent. Instead, its role, and that of
$\norm{\bar{\multipliers}}_1$, is played by the first coordinate of
$\bar{\multipliers}$. Inspection of \defref{proxy-lagrangians} reveals that, if
one or more of the constraints are violated, then the $\multipliers$-player
would prefer $\multipliers_1$ to be zero, whereas if they are satisfied (with
some margin), then it would prefer $\multipliers_1$ to be one. In other words,
the first coordinate of $\multipliers^{(t)}$ encodes the
$\multipliers$-player's belief about the feasibility of $\parameters^{(t)}$,
for which reason $\parameters^{(t)}$ is weighted by $\multipliers^{(t)}_1$ in
the density defining $\bar{\parameters}$.
%

\algref{stochastic-proxy-lagrangian} is motivated by the observation that,
while \thmref{proxy-lagrangian-suboptimality-and-feasibility} only requires
that the $\parameters^{(t)}$ sequence suffer low external regret \wrt
$\lagrangian_{\parameters}\left(\cdot, \multipliers^{(t)}\right)$, the
condition on the $\multipliers^{(t)}$ sequence is stronger, requiring it to
suffer low \emph{swap regret}~\citep{Blum:2007}.
%
%
Hence, the $\parameters$-player uses SGD to minimize external regret, while the
$\multipliers$-player uses a swap-regret minimization algorithm of the type
proposed by \citet{Gordon:2008}, yielding the convergence guarantee:
\input{theorems/lem-stochastic-proxy-lagrangian}
\algref{stochastic-proxy-lagrangian} is designed for the convex setting (except
for the $\constraint{i}$s), for which reason it uses SGD for the
$\parameters$-updates.
However, this convexity requirement is not innate to our approach: it's
straightforward to design an oracle-based algorithm that, like
\algref{oracle-lagrangian}, doesn't require convexity~\footnote{This is
\algref{oracle-proxy-lagrangian}, with \lemref{oracle-proxy-lagrangian} being
its convergence guarantee, both in \appref{convergence:two-player}.}. Our
reason for presenting the SGD-based algorithm, instead of the oracle-based one,
is that the purpose of proxy constraints is to substitute optimizable
constraints for unoptimizable ones, and there is no need to do so if you have
an oracle.

\subsection{Shrinking}\label{sec:proxy-lagrangian:shrinking}

\NOTE{we could move the lemma (but not the rest) to an appendix, if needed}
It turns out that the same existence result that we provided for the Lagrangian
game (\lemref{sparse-lagrangian})---of a \emph{Nash} equilibrium---holds for
the proxy-Lagrangian:
\input{theorems/lem-sparse-proxy-lagrangian}
Furthermore, the exact same linear programming procedure of
\lemref{sparse-linear-program} can be applied (with the $\vec{\constraint{i}}$s
being defined in terms of the \emph{original}---not proxy---constraints) to
yield a solution with support size $\numconstraints+1$, and works equally well.
This is easy to verify: since $\bar{\parameters}$, as defined in
\thmref{proxy-lagrangian-suboptimality-and-feasibility}, is a distribution over
the $\parameters^{(t)}$s, and is therefore feasible for the LP, the \emph{best}
distribution over the iterates will be at least as good.

%% file: figures/alg-stochastic-proxy-lagrangian.tex
\begin{algorithm*}[t]

\begin{pseudocode}
\codename $\mbox{StochasticProxyLagrangian}\left(\lagrangian_{\parameters}, \lagrangian_{\multipliers} : \Parameters \times \Delta^{\numconstraints+1} \rightarrow \R, T \in \N, \eta_{\parameters}, \eta_{\multipliers} \in \R_+ \right)$: \\
\codeline Initialize $\parameters^{(1)} = 0$, and $\matrixmultipliers^{(1)} \in \R^{\left(\numconstraints + 1\right) \times \left(\numconstraints + 1\right)}$ with $\matrixmultipliers_{i,j} = 1 / \left(\numconstraints+1\right)$ \codecomment{Assumes $0 \in \Parameters$} \\
\codeline For $t \in \indices{T}$: \\
\codeline \> Let $\multipliers^{(t)} = \fix \matrixmultipliers^{(t)}$ \codecomment{Stationary distribution of $\matrixmultipliers^{(t)}$} \\
\codeline \> Let $\stochasticsubgrad^{(t)}_{\parameters}$ be a stochastic subgradient of $\lagrangian_{\parameters}\left(\parameters^{(t)},\multipliers^{(t)}\right)$ \wrt $\parameters$ \\
\codeline \> Let $\stochasticgrad^{(t)}_{\multipliers}$ be a stochastic gradient of $\lagrangian_{\multipliers}\left(\parameters^{(t)},\multipliers^{(t)}\right)$ \wrt $\multipliers$ \\
\codeline \> Update $\parameters^{(t+1)} = \Pi_{\Parameters}\left( \parameters^{(t)} - \eta_{\parameters} \stochasticsubgrad^{(t)}_{\parameters} \right)$ \codecomment{Projected SGD update} \\
\codeline \> Update $\tilde{\matrixmultipliers}^{(t+1)} = \matrixmultipliers^{(t)} \elementwiseproduct \elementwiseexp\left( \eta_{\multipliers} \stochasticgrad^{(t)}_{\multipliers} \left( \multipliers^{(t)} \right)^T \right)$ \codecomment{$\elementwiseproduct$ and $\elementwiseexp$ are element-wise} \\
\codeline \> Project $\matrixmultipliers^{(t+1)}_{:,i} = \tilde{\matrixmultipliers}^{(t+1)}_{:,i} / \norm{\tilde{\matrixmultipliers}^{(t+1)}_{:,i}}_1$ for $i\in\indices{\numconstraints+1}$ \codecomment{Column-wise projection \wrt KL divergence} \\
\codeline Return $\parameters^{(1)},\dots,\parameters^{(T)}$ and $\multipliers^{(1)},\dots,\multipliers^{(T)}$
\end{pseudocode}

\caption{
  Optimizes the proxy-Lagrangian formulation (\defref{proxy-lagrangians}) in
  the convex setting, with the $\parameters$-player minimizing external regret,
  and the $\multipliers$-player minimizing swap regret.
  The $\fix \matrixmultipliers$ operation on line $3$ results in a stationary
  distribution of $\matrixmultipliers$ (\ie a $\multipliers \in \Multipliers$
  such that $\matrixmultipliers \multipliers = \multipliers$, which can be
  derived from the top eigenvector).
  The function $\Pi_{\Parameters}$ projects its argument onto $\Parameters$
  \wrt the Euclidean norm.
}

\label{alg:stochastic-proxy-lagrangian}

\end{algorithm*}

%% file: theorems/lem-stochastic-proxy-lagrangian.tex
\begin{lem}{stochastic-proxy-lagrangian}
  \ifshowproofs
  \textbf{(\algref{stochastic-proxy-lagrangian})}
  \fi
  Suppose that $\Parameters$ is a compact convex set, $\Matrixmultipliers$ and
  $\Multipliers$ are as in
  \thmref{proxy-lagrangian-suboptimality-and-feasibility}, and that the
  objective and proxy constraint functions
  $\objective,\proxyconstraint{1},\dots,\proxyconstraint{\numconstraints}$ are
  convex (but not $\constraint{1},\dots,\constraint{\numconstraints}$). Define
  the three upper bounds $\bound{\Parameters} \ge \max_{\parameters \in
  \Parameters} \norm{\parameters}_2$, $\bound{\stochasticsubgrad} \ge \max_{t
  \in \indices{T}} \norm{\stochasticsubgrad_{\parameters}^{(t)}}_2$, and
  $\bound{\stochasticgrad} \ge \max_{t \in \indices{T}}
  \norm{\stochasticgrad_{\multipliers}^{(t)}}_{\infty}$.

  If we run \algref{stochastic-proxy-lagrangian} with the step sizes
  $\eta_{\parameters} \defeq \bound{\Parameters} / \bound{\stochasticsubgrad}
  \sqrt{2T}$ and $\eta_{\multipliers} \defeq \sqrt{
  \left(\numconstraints+1\right) \ln \left(\numconstraints+1\right) / T
  \bound{\stochasticgrad}^2 }$, then the result satisfies the conditions of
  \thmref{proxy-lagrangian-suboptimality-and-feasibility} for:
  \begin{align*}
    \epsilon_{\parameters} =& 2 \bound{\Parameters} \bound{\stochasticsubgrad}
    \sqrt{ \frac{ 1 + 16 \ln\frac{2}{\delta} }{T} } \\
    \epsilon_{\multipliers} =& 2 \bound{\stochasticgrad} \sqrt{ \frac{ 2
    \left(\numconstraints+1\right) \ln \left(\numconstraints+1\right) \left( 1
    + 16 \ln\frac{2}{\delta}\right) }{T} }
  \end{align*}
  with probability $1-\delta$ over the draws of the stochastic
  (sub)gradients.
\end{lem}
\begin{prf}{stochastic-proxy-lagrangian}
  Applying \corref{stochastic-sgd} to the optimization over $\parameters$, and
  \lemref{stochastic-internal-regret} to that over $\multipliers$ (with
  $\matrixmultipliersize \defeq \numconstraints + 1$), gives that with
  probability $1-2\delta'$ over the draws of the stochastic
  (sub)gradients:
  \begin{align*}
    \frac{1}{T} \sum_{t=1}^T \lagrangian_{\parameters}\left( \parameters^{(t)},
    \multipliers^{(t)} \right) - \frac{1}{T} \sum_{t=1}^T
    \lagrangian_{\parameters}\left( \parameters^*, \multipliers^{(t)} \right)
    \le& 2 \bound{\Parameters} \bound{\stochasticsubgrad} \sqrt{ \frac{ 1 + 16
    \ln\frac{1}{\delta'} }{T} } \\
    \frac{1}{T} \sum_{t=1}^T \lagrangian_{\multipliers}\left(
    \parameters^{(t)}, \matrixmultipliers^* \multipliers^{(t)} \right) -
    \frac{1}{T} \sum_{t=1}^T \lagrangian_{\multipliers}\left(
    \parameters^{(t)}, \multipliers^{(t)} \right) \le& 2
    \bound{\stochasticgrad} \sqrt{ \frac{ 2 \left(\numconstraints+1\right)
    \ln \left(\numconstraints+1\right) \left( 1 + 16
    \ln\frac{1}{\delta'}\right) }{T} }
  \end{align*}
  Taking $\delta=2\delta'$, and using the definitions of $\bar{\parameters}$
  and $\bar{\multipliers}$, yields the claimed result.
\end{prf}

%% file: theorems/lem-sparse-proxy-lagrangian.tex
\begin{lem}{sparse-proxy-lagrangian}
  \NOTE{It would be nice to avoid the continuity assumption for the constraint
  functions---it comes from the minimax theorem~\citep{Glicksburg:1952}, but
  seems like it might be unnecessary}
  If $\Parameters$ is a compact Hausdorff space and the objective, constraint
  and proxy constraint functions
  $\objective,\constraint{1},\dots,\constraint{\numconstraints},\proxyconstraint{1},\dots,\proxyconstraint{\numconstraints}$
  are continuous, then the proxy-Lagrangian game (\defref{proxy-lagrangians})
  has a mixed Nash equilibrium pair $\left(\parameters,\multipliers\right)$
  where $\parameters$ is a random variable supported on at most
  $\numconstraints+1$ elements of $\Parameters$, and $\multipliers$ is
  non-random.
\end{lem}
\begin{prf}{sparse-proxy-lagrangian}
  Applying \thmref{sparse-equilibrium} directly would result in a support size
  of $\numconstraints+2$, rather than the desired $\numconstraints+1$, since
  $\Multipliers$ is $\left(\numconstraints+1\right)$-dimensional.
  Instead, we define $\tilde{\Multipliers} = \left\{ \tilde{\multipliers} \in
  \R_+^{\numconstraints} : \norm{\tilde{\multipliers}}_1 \le 1 \right\}$ as the
  space containing the last $\numconstraints$ coordinates of $\Multipliers$.
  Then we can rewrite the proxy-Lagrangian functions
  $\tilde{\lagrangian}_{\parameters},\tilde{\lagrangian}_{\multipliers} :
  \Parameters \times \tilde{\Multipliers} \rightarrow \R$ as:
  \begin{align*}
    \tilde{\lagrangian}_{\parameters}\left(\parameters,
    \tilde{\multipliers}\right) =& \left(1 -
    \norm{\tilde{\multipliers}}_1\right) \objective\left(\parameters\right) +
    \sum_{i=1}^{\numconstraints} \tilde{\multipliers}_i
    \proxyconstraint{i}\left(\parameters\right) \\
    \tilde{\lagrangian}_{\multipliers}\left(\parameters,
    \tilde{\multipliers}\right) =& \sum_{i=1}^{\numconstraints}
    \tilde{\multipliers}_i \constraint{i}\left(\parameters\right)
  \end{align*}
  These functions are linear in $\tilde{\multipliers}$, which is a
  $\numconstraints$-dimensional space, so the conditions of
  \thmref{sparse-equilibrium} apply, yielding the claimed result.
\end{prf}

%% file: sec-overall.tex
\section{Overall Procedure}\label{sec:overall}

The pieces are now in place to propose a complete two-phase optimization
procedure, for both convex and non-convex problems, with or without proxy
constraints.
In the first phase, we apply the appropriate algorithm to yield a distribution
over the $T$ ``candidates'' $\parameters^{(1)},\dots,\parameters^{(T)}$ that is
approximately feasible and optimal, according to either
\thmrefs[or]{lagrangian-suboptimality-and-feasibility}{proxy-lagrangian-suboptimality-and-feasibility}.
Then, in the second phase, we construct
$\vec{\objective},\vec{\constraint{1}},\dots,\vec{\constraint{\numconstraints}}
\in \R^T$ by evaluating the objective and constraint functions for each
$\parameters^{(t)}$, and then optimize the LP of \lemref{sparse-linear-program}
to find the \emph{best} distribution over
$\parameters^{(1)},\dots,\parameters^{(T)}$ (which will have support size $\le
\numconstraints + 1$). If we take the $\epsilon$ parameter to this LP to be
either the RHS of
\eqref{thm:lagrangian-suboptimality-and-feasibility:feasibility} in
\thmref{lagrangian-suboptimality-and-feasibility} (for the Lagrangian case), or
of \eqref{thm:proxy-lagrangian-suboptimality-and-feasibility:feasibility} in
\thmref{proxy-lagrangian-suboptimality-and-feasibility} (for the
proxy-Lagrangian case), then the resulting
size-$\left(\numconstraints+1\right)$ distribution will have the same
guarantees as the original.

\begin{titled-paragraph}{Practical Procedure}
The approach outlined above provably works, but is still somewhat idealized.
In practice, we'll dispense with the oracle $\oracle$---even on non-convex
problems---in favor of the ``typical'' approach: pretending that the problem is
convex, and using SGD (or another cheap stochastic algorithm) for the
$\parameters$-updates~\footnote{In the Lagrangian case, this is
\algref{stochastic-lagrangian}, with \lemref{stochastic-lagrangian} being its
convergence guarantee in the convex setting, both in
\appref{convergence:two-player}. In the proxy-Lagrangian case, this is
\algref{stochastic-proxy-lagrangian}.}.
On a non-convex problem, this has no guarantees, but one would still hope that
it would result in a ``candidate set'' of $\parameters^{(t)}$s that contains
enough good solutions to pass on to the LP of \lemref{sparse-linear-program}.
If necessary, this candidate set can first be subsampled to make it a
reasonable size.
To choose the $\epsilon$ parameter of the LP, we propose using a bisection
search to find the smallest $\epsilon \ge 0$ for which there exists a feasible
solution.
\end{titled-paragraph}

\begin{titled-paragraph}{Evaluation}
The ultimate result of either of these procedures is a distribution over at
most $\numconstraints+1$ distinct $\parameters$s. If the underlying problem is
one of classification, with $\classifier\left(\cdot;\parameters\right)$ being
the scoring function, then this distribution defines a stochastic classifier:
at evaluation time, upon receiving an example $x$, we would sample
$\parameters$, and then return $\classifier\left(x;\parameters\right)$.
If a stochastic classifier is not acceptable (as is often the case in
real-world applications), then one could heuristically convert it into a
deterministic one, \eg by weighted averaging or voting, which is made
significantly easier by its small size.
\end{titled-paragraph}

%% file: sec-experiments.tex
\section{Experiments}\label{sec:experiments}

\input{figures/tab-experiments-mnist}
\input{figures/tab-experiments-adult}
We present two experiments: the first, on the robust MNIST problem of
\citet{Chen:2017}, tests the performance of the ``practical procedure'' of
\secref{overall} using the Lagrangian formulation (with the norms of the
Lagrange multipliers being unbounded, \ie $\Radius=\infty$), while the second,
a fairness problem on the UCI Adult dataset~\citep{UCI}, uses the
proxy-Lagrangian formulation. Both were implemented in
\tensorflow~\footnote{Source code:
\url{https://github.com/tensorflow/tensorflow/tree/r1.10/tensorflow/contrib/constrained_optimization}.}.

In both cases, the $\parameters$ and $\multipliers$-updates both used \adagrad
with the same initial learning rates. In the proxy-Lagrangian case, however,
the $\multipliers$-update (line $7$ of \algref{stochastic-proxy-lagrangian})
was performed in the log domain so that it would be multiplicative. To choose
the initial \adagrad learning rate, we performed a grid search over
powers-of-two, and chose the best model on a validation set. In all
experiments, the optimum was in the interior of the grid.

Our constrained optimization algorithms result in stochastic classifiers, and
we report results for \emph{both} the $\bar{\parameters}$ of
\thmrefs[or]{lagrangian-suboptimality-and-feasibility}{proxy-lagrangian-suboptimality-and-feasibility}
(in the Lagrangian or proxy-Lagrangian cases, respectively), \emph{and} the
optimal distribution found by the LP of \lemref{sparse-linear-program},
optimized on the training dataset.

\subsection{Robust Optimization}\label{sec:experiments:mnist}

In robust optimization, there are multiple objective functions
$\constraint{1},...,\constraint{m} : \Parameters \rightarrow \R$, and the goal
is to find a $\parameters \in \Parameters$ minimizing $\max_{i \in
\indices{\numconstraints}} \constraint{i}\left(\parameters\right)$. As was
discussed in \secref{related-work}, this can be re-written as a constrained
problem by introducing a slack variable, as in \eqref{related-work:robust}.

The task is the modified MNIST problem created by \citet{Chen:2017}, which is
based on four datasets, each of which is a version of MNIST that has been
corrupted in different ways. One would therefore hope that choosing
$\constraint{i}$ to be an empirical loss on the $i$th such dataset, and
optimizing the corresponding robust problem, will result in a classifier that
is ``robust'' to all four types of corruption.

We used a neural network with one $1024$-neuron hidden layer, and ReLu
activations. The four objective functions were the cross-entropy losses on the
corrupted datasets.
All models were trained for $50\,000$ iterations using a minibatch size of
$100$, and a $\parameters^{(t)}$ was extracted every $500$ iterations, yielding
a sequence of length $T=100$.

\begin{titled-paragraph}{Baselines}
For our baselines, we trained the neural network over the union of the four
datasets.
We report two variants: (i) the ``Uniform Distribution Baseline'' of
\citet{Chen:2017} is a stochastic classifier, uniformly sampled over the
$\parameters^{(t)}$s (like our $\bar{\parameters}$ classifier), and (ii) a
non-stochastic classifier taking its parameters from the last iterate
$\parameters^{(T)}$.
\end{titled-paragraph}

\begin{titled-paragraph}{Results}
\tabref{experiments-mnist} lists, for each of the corrupted datasets, the error
rates of the compared models on both the training and testing datasets.
Interestingly, although our proposed shrinking procedure is only guaranteed to
give a distribution over $m+1$ solutions, in these experiments it chose only
one.
Hence, the ``Lagrangian (LP)'' model of \tabref{experiments-mnist} is, like
``Baseline ($\parameters^{(T)}$)'', non-stochastic.

While we did not quite match the raw performance reported by
\citet{Chen:2017}'s algorithm, our results, and theirs, tell similar stories.
In particular, we can see that both of our algorithms outperformed their
natural baseline equivalents. In particular, the use of shrinking not only
greatly simplified the model, but also significantly improved performance.
\end{titled-paragraph}

\subsection{Equal Opportunity}\label{sec:experiments:adult}

These experiments were performed on the UCI Adult dataset, which consists of
census data including $14$ features such as age, gender, race, occupation, and
education. The goal was to predict whether income exceeds 50k/year. The dataset
contains $32\,561$ training examples, from which we split off 20\% to form a
validation set, and $16\,281$ testing examples.

We dropped the ``fnlwgt'' weighting feature, and processed the remaining
features as in \citet{Platt:1998}, yielding $120$ binary features, on which we
trained linear models. The objective was to minimize the average hinge loss,
subject to one 95\% equal opportunity~\citep{Hardt:2016} constraint in the
style of \citet{Goh:2016} for each ``protected class'': $\constraint{i}$ was
defined such that $\constraint{i}\left(\parameters\right) \le 0$ iff the
positive prediction rate on the set of positively-labeled examples for the
associated class was at least 95\% of the positive prediction rate on the set
of all positively-labeled examples.

When using proxy constraints, $\proxyconstraint{i}$ was taken to be a version
of $\constraint{i}$ with the indicator functions defining the positive
prediction rates replaced with hinge upper bounds. When not using proxy
constraints, the indicator-based constraints were dropped entirely, with these
upper bounds being used throughout.

All models were trained for $5\,000$ iterations with a minibatch size of $100$,
with a $\parameters^{(t)}$ being extracted every $50$ iterations, yielding a
sequence of length $T=100$.

\begin{titled-paragraph}{Baseline}
The baseline classifier was optimized to simply minimize training hinge loss.
Since this problem is unconstrained, we took the last iterate
$\parameters^{(T)}$.
\end{titled-paragraph}

\begin{titled-paragraph}{``Best-model'' Heuristic}
For hyperparameter tuning using a grid search, we needed to choose the ``best''
model on the validation set. Due to the presence of constraints, however, the
``best'' model was not necessarily that with the lowest validation error.
Instead, we used the following heuristic: the models were each ranked in terms
of their objective function value, as well as the magnitude of the $i$th
constraint violation (\ie $\max\left\{0,
\constraint{i}\left(\parameters\right)\right\}$). The ``score'' of each model
was then taken to be the maximal such rank, and the model with the lowest score
was chosen, with the objective function serving as a tiebreaker.
%
%
\end{titled-paragraph}

\begin{titled-paragraph}{Results}
\tabref{experiments-adult} lists the test error rates, (indicator-based)
constraint function values on both the training and testing datasets, and
support sizes of the stochastic classifiers, for each of the compared
algorithms.
%
%
The ``LP'' versions of our models, which were found using the shrinking
procedure of \lemref{sparse-linear-program}, uniformly outperformed their
$\bar{\parameters}$-analogues. We can see, however, that the generalization
issue discussed in \secref{related-work:alternatives} caused the
proxy-Lagrangian LP model to slightly violate the constraints on the testing
dataset, despite satisfying them on the training dataset.
The non-proxy algorithms satisfied all constraints, on both the training and
testing datasets, because there was sufficient ``room'' between the hinge upper
bound that they actually constrained, and the true constraint, to absorb the
generalization error. Inspection of the error rates, however, reveals that the
relaxed constraints were so overly-conservative that satisfying them
significantly damaged classification performance. In contrast, our
proxy-Lagrangian approach matched the classification performance of the
unconstrained baseline.
\end{titled-paragraph}

%% file: figures/tab-experiments-mnist.tex
\begin{table*}[t]

\centering

\caption{
  Error rates on the experiments of \secref{experiments:mnist}. The columns
  correspond to the four corrupted datasets of \citet{Chen:2017}.
}

\begin{tabular}{r|cccc|cccc}
  \toprule
  & \multicolumn{4}{c|}{Testing} & \multicolumn{4}{c}{Training} \\
  & Set 1 & Set 2 & Set 3 & Set 4 &
  Set 1 & Set 2 & Set 3 & Set 4\\
  \midrule
  Baseline ($\bar{\parameters}$) &
  2.58 & 2.66 & 2.01 & 2.52 &
  1.06 & 1.45 & 0.43 & 1.10 \\
  Baseline ($\parameters^{(T)}$) &
  1.77 & 1.92 & 1.77 & 1.75 &
  0.04 & 0.40 & 0.01 & 0.08 \\
  Lagrangian ($\bar{\parameters}$) &
  2.04 & 2.15 & 1.96 & 2.04 &
  0.42 & 0.70 & 0.30 & 0.49 \\
  Lagrangian (LP) &
  1.66 & 1.67 & 1.63 & 1.62 &
  0.00 & 0.01 & 0.00 & 0.00 \\
  \bottomrule
\end{tabular}

\label{tab:experiments-mnist}

\end{table*}

%% file: figures/tab-experiments-adult.tex
\begin{table*}[t]

\centering

\caption{
  Support sizes, test error rates, and ``equal opportunity'' values for the
  experiments of \secref{experiments:adult}. For the constraints, each reported
  number is the ratio of the positive prediction rate on positively-labeled
  members of the protected class, to the positive prediction rate on the set of
  all positively-labeled data. The constraints attempt to force this ratio to
  be at least $95\%$---quantities lower than this threshold violate the
  constraint, and are marked in \textbf{bold}.
}

\begin{tabular}{r|c|c|cccc|cccc}
  \toprule
  & & \multicolumn{5}{c|}{Testing} & \multicolumn{4}{c}{Training} \\
  & Support & Error & Female & Male & Black & White & Female & Male & Black & White \\
  \midrule
  Baseline ($\parameters^{(T)}$) & 1 &
  14.2\% & \textbf{89.5\%} & 102\% & \textbf{81.6\%} & 101\%  &
  \textbf{88.9\%} & 102\% & \textbf{82.8\%} & 101\% \\
  Lagrangian ($\bar{\parameters}$) & 100 &
  16.3\% & 114\% & 97.5\% & 126\% & 99.8\% &
  113\% & 97.8\% & 121\% & 99.7\% \\
  Lagrangian (LP) & 3 &
  15.5\% & 106\% & 99.0\% & 111\% & 101\% &
  104\% & 99.4\% & 105\% & 101\% \\
  Proxy ($\bar{\parameters}$) & 100 &
  14.4\% & \textbf{94.1\%} & 101\% & \textbf{94.9\%} & 100\% &
  \textbf{94.7\%} & 101\% & \textbf{94.5\%} & 100\% \\
  Proxy (LP) & 3 &
  14.2\% & \textbf{94.4\%} & 101\% & \textbf{94.9\%} & 100\% &
  95.0\% & 101\% & 95.0\% & 100\% \\
  \bottomrule
\end{tabular}

\label{tab:experiments-adult}

\end{table*}

%% file: sec-acknowledgements.tex
\section*{Acknowledgments}

We thank Seungil You for initially posing the question of whether constraint
functions could be relaxed for only the $\parameters$-player, as well as Maya
Gupta, Taman Narayan and Serena Wang for helping to develop the heuristic used
to choose the ``best'' model on the validation dataset in
\secref{experiments:adult}.

%% file: app-proofs.tex
\section{Proofs of Sub\{optimality,feasibility\} Guarantees}\label{app:suboptimality}

\input{theorems/thm-lagrangian-suboptimality}
\input{theorems/lem-lagrangian-feasibility}
\input{theorems/thm-proxy-lagrangian-suboptimality}
\input{theorems/lem-proxy-lagrangian-feasibility}

\section{Proofs of Existence of Sparse Equilibria}\label{app:sparsity}

\input{theorems/thm-sparse-equilibrium}

\input{theorems/lem-sparse-proxy-lagrangian}
\input{theorems/lem-sparse-linear-program}
\section{Proofs of Convergence Rates}\label{app:convergence}

\subsection{Non-Stochastic One-Player Convergence Rates}

\input{theorems/thm-mirror}
\input{theorems/cor-sgd}
\input{theorems/cor-matrix-multiplicative}
\input{theorems/lem-internal-regret}

\subsection{Stochastic One-Player Convergence Rates}

\input{theorems/thm-stochastic-mirror}
\input{theorems/cor-stochastic-sgd}
\input{theorems/cor-stochastic-matrix-multiplicative}
\input{theorems/lem-stochastic-internal-regret}

\subsection{Two-Player Convergence Rates}\label{app:convergence:two-player}

\input{theorems/lem-oracle-lagrangian}
\input{figures/alg-stochastic-lagrangian}
\input{theorems/lem-stochastic-lagrangian}
\input{figures/alg-oracle-proxy-lagrangian}
\input{theorems/lem-oracle-proxy-lagrangian}

\input{theorems/lem-stochastic-proxy-lagrangian}

%% file: theorems/thm-lagrangian-suboptimality.tex
\begin{thm}{lagrangian-suboptimality}
  \textbf{(Lagrangian Sub\{optimality,feasibility\})}
  Define $\Multipliers = \left\{ \multipliers \in \R_+^{\numconstraints} :
  \norm{\multipliers}_p \le \Radius \right\}$, and consider the Lagrangian of
  \eqref{constrained-problem} (\defref{lagrangian}).
  Suppose that $\parameters \in \Parameters$ and $\multipliers \in
  \Multipliers$ are random variables such that:
  \begin{equation}
    \replabel{eq:thm:lagrangian-suboptimality:epsilon}
    \max_{\multipliers^* \in \Multipliers} \expectation_{\parameters}\left[
    \lagrangian\left( \parameters, \multipliers^* \right) \right] -
    \inf_{\parameters^* \in \Parameters} \expectation_{\multipliers}\left[
    \lagrangian\left( \parameters^*, \multipliers \right) \right] \le \epsilon
  \end{equation}
  \ie $\parameters,\multipliers$ is an $\epsilon$-approximate Nash equilibrium.
  Then $\parameters$ is $\epsilon$-suboptimal:
  \begin{equation*}
    \expectation_{\parameters}\left[ \objective\left(\parameters\right) \right]
    \le \inf_{\parameters^* \in \Parameters : \forall i \in
    \indices{\numconstraints} . \constraint{i}\left(\parameters^*\right) \le 0}
    \objective\left(\parameters^*\right) + \epsilon
  \end{equation*}
  Furthermore, if $\multipliers$ is in the interior of $\Multipliers$, in the
  sense that $\norm{\bar{\multipliers}}_p < \Radius$ where $\bar{\multipliers}
  \defeq \expectation_{\multipliers}\left[ \multipliers \right]$, then
  $\parameters$ is $\epsilon / \left(\Radius -
  \norm{\bar{\multipliers}}_p\right)$-feasible:
  \begin{equation*}
    \norm{\left( \expectation_{\parameters} \left[
    \constraint{:}\left(\parameters\right) \right] \right)_+}_q \le
    \frac{\epsilon}{\Radius - \norm{\bar{\multipliers}}_p}
  \end{equation*}
  where $\constraint{:}\left(\parameters\right)$ is the
  $\numconstraints$-dimensional vector of constraint evaluations, and
  $\left(\cdot\right)_+$ takes the positive part of its argument, so that
  $\norm{\left( \expectation_{\parameters} \left[
  \constraint{:}\left(\parameters\right) \right] \right)_+}_q$ is the $q$-norm
  of the vector of expected constraint violations.
\end{thm}
\begin{prf}{lagrangian-suboptimality}
  First notice that $\lagrangian$ is linear in $\multipliers$, so:
  \begin{equation}
    \label{eq:thm:lagrangian-suboptimality:equilibrium}
    \max_{\multipliers^* \in \Multipliers} \expectation_{\parameters}\left[
    \lagrangian\left( \parameters, \multipliers^* \right) \right] -
    \inf_{\parameters^* \in \Parameters} \lagrangian\left( \parameters^*,
    \bar{\multipliers} \right) \le \epsilon
  \end{equation}

  \begin{titled-paragraph}{Optimality}
    Choose $\parameters^*$ to be the optimal \emph{feasible} solution in
    \eqref{thm:lagrangian-suboptimality:equilibrium}, so that
    $\constraint{i}\left(\parameters^*\right) \le 0$ for all
    $i\in\indices{\numconstraints}$, and also choose $\multipliers^* = 0$, which
    combined with the definition of $\lagrangian$
    (\defref{lagrangian}) gives that:
    \begin{equation*}
      \expectation_{\parameters}\left[ \objective\left(\parameters\right) \right]
      - \objective\left(\parameters^*\right) \le \epsilon
    \end{equation*}
    which is the optimality claim.
  \end{titled-paragraph}

  \begin{titled-paragraph}{Feasibility}
    Choose $\parameters^* = \parameters$ in
    \eqref{thm:lagrangian-suboptimality:equilibrium}. By the definition of $\lagrangian$
    (\defref{lagrangian}):
    \begin{equation*}
      \max_{\multipliers^* \in \Multipliers} \sum_{i=1}^{\numconstraints}
      \multipliers_i^* \expectation_{\parameters} \left[
      \constraint{i}\left(\parameters\right) \right] -
      \sum_{i=1}^{\numconstraints} \bar{\multipliers}_i
      \expectation_{\parameters} \left[ \constraint{i}\left(\parameters\right)
      \right] \le \epsilon
    \end{equation*}
    Then by the definition of a dual norm, H\"older's inequality, and the
    assumption that $\norm{\bar{\multipliers}}_p < \Radius$:
    \begin{equation*}
      \Radius \norm{\left( \expectation_{\parameters} \left[
      \constraint{:}\left(\parameters\right) \right] \right)_+}_q - \norm{\bar{\multipliers}}_p
      \norm{ \left( \expectation_{\parameters} \left[
      \constraint{:}\left(\parameters\right) \right] \right)_+ }_q \le \epsilon
    \end{equation*}
    Rearranging terms gives the feasibility claim.
  \end{titled-paragraph}
\end{prf}

%% file: theorems/lem-lagrangian-feasibility.tex
\begin{lem}{lagrangian-feasibility}
  In the context of \thmref{lagrangian-suboptimality}, suppose that there
  exists a $\parameters' \in \Parameters$ that satisfies all of the
  constraints, and does so with $q$-norm margin $\margin$, \ie
  $\constraint{i}\left(\parameters'\right) \le 0$ for all $i \in
  \indices{\numconstraints}$ and
  $\norm{\constraint{:}\left(\parameters'\right)}_q \ge \margin$. Then:
  \begin{equation*}
    \norm{ \bar{\multipliers} }_p \le \frac{\epsilon +
    \bound{\objective}}{\gamma}
  \end{equation*}
  where $\bound{\objective} \ge \sup_{\parameters \in \Parameters}
  \objective\left(\parameters\right) - \inf_{\parameters \in \Parameters}
  \objective\left(\parameters\right)$ is a bound on the range of the objective
  function $\objective$.
\end{lem}
\begin{prf}{lagrangian-feasibility}
  Starting from \eqref{thm:lagrangian-suboptimality:epsilon} (in
  \thmref{lagrangian-suboptimality}), and choosing $\parameters^* =
  \parameters'$ and $\multipliers^* = 0$:
  \begin{align*}
    \epsilon \ge& \expectation_{\parameters}\left[
    \objective\left(\parameters\right) \right] -
    \expectation_{\multipliers}\left[ \objective\left(\parameters'\right) +
    \sum_{i=1}^{\numconstraints} \multipliers_i
    \constraint{i}\left(\parameters'\right) \right] \\
    \epsilon \ge& \expectation_{\parameters}\left[
    \objective\left(\parameters\right) - \inf_{\parameters' \in \Parameters}
    \objective\left(\parameters'\right) \right] - \left(
    \objective\left(\parameters'\right) - \inf_{\parameters' \in \Parameters}
    \objective\left(\parameters'\right) \right) + \gamma \norm{
    \bar{\multipliers} }_p \\
    \epsilon \ge& - \bound{\objective} + \gamma \norm{ \bar{\multipliers} }_p
  \end{align*}
  Solving for $\norm{ \bar{\multipliers} }_p$ yields the claim.
\end{prf}

%% file: theorems/thm-proxy-lagrangian-suboptimality.tex
\begin{thm}{proxy-lagrangian-suboptimality}
  \textbf{(Proxy-Lagrangian Sub\{optimality,feasibility\})}
  Let $\Matrixmultipliers$ be the set of all left-stochastic
  $\left(\numconstraints + 1\right) \times \left(\numconstraints + 1\right)$
  matrices (\ie $\Matrixmultipliers \defeq \left\{ \matrixmultipliers \in
  \R^{\left(m+1\right)\times\left(m+1\right)} : \forall i \in
  \indices{\numconstraints + 1} . \matrixmultipliers_{:, i} \in
  \Delta^{\numconstraints + 1} \right\}$), and consider the
  ``proxy-Lagrangians'' of \eqref{constrained-problem}
  (\defref{proxy-lagrangians}).
  Suppose that $\parameters \in \Parameters$ and $\multipliers \in
  \Multipliers$ are jointly distributed random variables such that:
  \begin{align}
    \replabel{eq:thm:proxy-lagrangian-suboptimality:equilibrium}
    \expectation_{\parameters,\multipliers}\left[
    \lagrangian_{\parameters}\left( \parameters, \multipliers \right) \right] -
    \inf_{\parameters^* \in \Parameters} \expectation_{\multipliers}\left[
    \lagrangian_{\parameters}\left( \parameters^*, \multipliers \right) \right]
    \le& \epsilon_{\parameters} \\
    \notag \max_{\matrixmultipliers^* \in \Matrixmultipliers}
    \expectation_{\parameters,\multipliers}\left[
    \lagrangian_{\multipliers}\left( \parameters, \matrixmultipliers^*
    \multipliers \right) \right] -
    \expectation_{\parameters,\multipliers}\left[
    \lagrangian_{\multipliers}\left( \parameters, \multipliers \right) \right]
    \le& \epsilon_{\multipliers}
  \end{align}
  \TODO{this is cumbersome}
  Define $\bar{\multipliers} \defeq \expectation_{\multipliers}\left[
  \multipliers \right]$, let $\left( \Omega, \mathcal{F}, P \right)$ be the
  probability space, and define a random variable $\bar{\parameters}$ such
  that:
  \begin{equation*}
    \probability\left\{ \bar{\parameters} \in S \right\} = \frac{
    \int_{\parameters^{-1}\left(S\right)} \multipliers_1\left(x\right)
    dP\left(x\right) }{ \int_{\Omega} \multipliers_1\left(x\right)
    dP\left(x\right) }
  \end{equation*}
  In words, $\bar{\parameters}$ is a version of $\parameters$ that has been
  resampled with $\multipliers_1$ being treated as an importance weight. In
  particular $\expectation_{\bar{\parameters}}\left[
  f\left(\bar{\parameters}\right) \right] =
  \expectation_{\parameters,\multipliers}\left[ \multipliers_1
  f\left(\parameters\right) \right] / \bar{\multipliers}_1$ for any
  $f:\Parameters \rightarrow \R$.
  Then $\bar{\parameters}$ is nearly-optimal:
  \begin{equation*}
    \expectation_{\bar{\parameters}}\left[
    \objective\left(\bar{\parameters}\right) \right] \le \inf_{\parameters^*
    \in \Parameters : \forall i \in \indices{\numconstraints} .
    \proxyconstraint{i}\left(\parameters^*\right) \le 0} \objective\left(
    \parameters^* \right) + \frac{\epsilon_{\parameters} +
    \epsilon_{\multipliers}}{\bar{\multipliers}_1}
  \end{equation*}
  and nearly-feasible:
  \begin{equation*}
    \norm{ \left( \expectation_{\bar{\parameters}}\left[
    \constraint{:}\left(\bar{\parameters}\right) \right] \right)_+ }_{\infty}
    \le \frac{\epsilon_{\multipliers}}{\bar{\multipliers}_1}
  \end{equation*}
  %
  %
  Notice the optimality inequality is weaker than it may appear, since the
  comparator in this equation is \emph{not} the optimal solution \wrt the
  constraints $\constraint{i}$, but rather \wrt the \emph{proxy} constraints
  $\proxyconstraint{i}$.
\end{thm}
\begin{prf}{proxy-lagrangian-suboptimality}
  \begin{titled-paragraph}{Optimality}
    If we choose $M^*$ to be the matrix with its first row being all-one, and
    all other rows being all-zero, then $\lagrangian_{\multipliers}\left(
    \parameters, \matrixmultipliers^* \multipliers \right) = 0$, which shows
    that the first term in the LHS of the second line of
    \eqref{thm:proxy-lagrangian-suboptimality:equilibrium} is nonnegative.
    Hence, $- \expectation_{\parameters,\multipliers}\left[
    \lagrangian_{\multipliers}\left( \parameters, \multipliers \right) \right]
    \le \epsilon_{\multipliers}$, so by the definition of
    $\lagrangian_{\multipliers}$ (\defref{proxy-lagrangians}), and the fact
    that $\proxyconstraint{i} \ge \constraint{i}$:
    \begin{equation*}
      \expectation_{\parameters,\multipliers}\left[
      \sum_{i=1}^{\numconstraints} \multipliers_{i+1}
      \proxyconstraint{i}\left(\parameters\right) \right] \ge
      -\epsilon_{\multipliers}
    \end{equation*}
    Notice that $\lagrangian_{\parameters}$ is linear in $\multipliers$, so the
    first line of \eqref{thm:proxy-lagrangian-suboptimality:equilibrium},
    combined with the above result and the definition of
    $\lagrangian_{\parameters}$ (\defref{proxy-lagrangians}) becomes:
    \begin{equation}
      \label{eq:thm:proxy-lagrangian-suboptimality:both-epsilons}
      \expectation_{\parameters,\multipliers}\left[ \multipliers_1
      \objective\left(\parameters\right) \right] - \inf_{\parameters^* \in
      \Parameters} \left( \bar{\multipliers}_1
      \objective\left(\parameters^*\right) + \sum_{i=1}^{\numconstraints}
      \bar{\multipliers}_{i+1} \proxyconstraint{i}\left(\parameters^*\right)
      \right) \le \epsilon_{\parameters} + \epsilon_{\multipliers}
    \end{equation}
    Choose $\parameters^*$ to be the optimal solution that satisfies the
    \emph{proxy} constraints $\tilde{g}$, so that
    $\proxyconstraint{i}\left(\parameters^*\right) \le 0$ for all
    $i\in\indices{\numconstraints}$. Then:
    \begin{equation*}
      \expectation_{\parameters,\multipliers}\left[ \multipliers_1
      \objective\left(\parameters\right) \right] - \bar{\multipliers}_1
      \objective\left( \parameters^* \right) \le \epsilon_{\parameters} +
      \epsilon_{\multipliers}
    \end{equation*}
    which is the optimality claim.
  \end{titled-paragraph}

  \begin{titled-paragraph}{Feasibility}
    We'll simplify our notation by defining $\ell_1\left(\parameters\right)
    \defeq 0$ and $\ell_{i+1}\left(\parameters\right) \defeq
    \constraint{i}\left(\parameters\right)$ for
    $i\in\indices{\numconstraints}$, so that
    $\lagrangian_{\multipliers}\left(\parameters, \multipliers\right) =
    \inner{\multipliers}{\ell_{:}\left(\parameters\right)}$.
    Consider the first term in the LHS of the second line of
    \eqref{thm:proxy-lagrangian-suboptimality:equilibrium}:
    \begin{align*}
      \max_{\matrixmultipliers^* \in \Matrixmultipliers}
      \expectation_{\parameters,\multipliers}\left[
      \lagrangian_{\multipliers}\left( \parameters, \matrixmultipliers^*
      \multipliers \right) \right] =&
      \max_{\matrixmultipliers^* \in \Matrixmultipliers}
      \expectation_{\parameters,\multipliers}\left[ \inner{\matrixmultipliers^*
      \multipliers}{\ell_{:}\left(\parameters\right)} \right] \\
      =& \max_{\matrixmultipliers^* \in \Matrixmultipliers}
      \expectation_{\parameters,\multipliers}\left[
      \sum_{i=1}^{\numconstraints+1} \sum_{j=1}^{\numconstraints+1}
      \matrixmultipliers^*_{j,i} \multipliers_i \ell_j\left(\parameters\right)
      \right] \\
      =& \sum_{i=1}^{\numconstraints+1} \max_{\matrixmultipliers^*_{:,i} \in
      \Delta^{\numconstraints+1}} \sum_{j=1}^{\numconstraints+1}
      \expectation_{\parameters,\multipliers}\left[ \matrixmultipliers^*_{j,i}
      \multipliers_i \ell_j\left(\parameters\right) \right] \\
      =& \sum_{i=1}^{\numconstraints+1} \max_{j \in
      \indices{\numconstraints+1}}
      \expectation_{\parameters,\multipliers}\left[ \multipliers_i
      \ell_j\left(\parameters\right) \right]
    \end{align*}
    where we used the fact that, since $\matrixmultipliers^*$ is
    left-stochastic, each of its columns is a
    $\left(\numconstraints+1\right)$-dimensional multinoulli distribution.
    For the second term in the LHS of the second line of
    \eqref{thm:proxy-lagrangian-suboptimality:equilibrium}, we can use the fact
    that $\ell_1\left(\parameters\right) = 0$:
    \begin{equation*}
      \expectation_{\parameters,\multipliers}\left[
      \sum_{i=2}^{\numconstraints+1} \multipliers_i
      \ell_i\left(\parameters\right) \right] \le \sum_{i=2}^{\numconstraints+1}
      \max_{j \in \indices{\numconstraints+1}}
      \expectation_{\parameters,\multipliers}\left[ \multipliers_i
      \ell_j\left(\parameters\right) \right]
    \end{equation*}
    Plugging these two results into the second line of
    \eqref{thm:proxy-lagrangian-suboptimality:equilibrium}, the two sums
    collapse, leaving:
    \begin{equation*}
      \max_{i \in \indices{\numconstraints+1}}
      \expectation_{\parameters,\multipliers}\left[ \multipliers_1
      \ell_i\left(\parameters\right) \right] \le \epsilon_{\multipliers}
    \end{equation*}
    The definition of $\ell_i$ then yields the feasibility claim.
  \end{titled-paragraph}
\end{prf}

%% file: theorems/lem-proxy-lagrangian-feasibility.tex
\begin{lem}{proxy-lagrangian-feasibility}
  In the context of \thmref{proxy-lagrangian-suboptimality}, suppose that there
  exists a $\parameters' \in \Parameters$ that satisfies all of the
  \emph{proxy} constraints with margin $\margin$, \ie
  $\proxyconstraint{i}\left(\parameters'\right) \le -\margin$ for all $i \in
  \indices{\numconstraints}$. Then:
  \begin{equation*}
    \bar{\multipliers}_1 \ge \frac{\gamma - \epsilon_{\parameters} -
    \epsilon_{\multipliers}}{\gamma + \bound{\objective}}
  \end{equation*}
  where $\bound{\objective} \ge \sup_{\parameters \in \Parameters}
  \objective\left(\parameters\right) - \inf_{\parameters \in \Parameters}
  \objective\left(\parameters\right)$ is a bound on the range of the objective
  function $\objective$.
\end{lem}
\begin{prf}{proxy-lagrangian-feasibility}
  Starting from \eqref{thm:proxy-lagrangian-suboptimality:both-epsilons} (in
  the proof of \thmref{proxy-lagrangian-suboptimality}), and choosing
  $\parameters^* = \parameters'$:
  \begin{equation*}
    \expectation_{\parameters,\multipliers}\left[ \multipliers_1
    \objective\left(\parameters\right) \right] - \left( \bar{\multipliers}_1
    \objective\left(\parameters'\right) + \sum_{i=1}^{\numconstraints}
    \bar{\multipliers}_{i+1} \proxyconstraint{i}\left(\parameters'\right)
    \right) \le \epsilon_{\parameters} + \epsilon_{\multipliers}
  \end{equation*}
  Since $\proxyconstraint{i}\left(\parameters'\right) \le -\margin$ for all $i
  \in \indices{\numconstraints}$:
  \begin{align*}
    \epsilon_{\parameters} + \epsilon_{\multipliers} \ge &
    \expectation_{\parameters,\multipliers}\left[ \multipliers_1
    \objective\left(\parameters\right) \right] - \bar{\multipliers}_1
    \objective\left(\parameters'\right) + \left(1 - \bar{\multipliers}_1\right)
    \gamma \\
    \ge & \expectation_{\parameters,\multipliers}\left[ \multipliers_1 \left(
    \objective\left(\parameters\right) - \inf_{\parameters' \in \Parameters}
    \objective\left(\parameters'\right) \right) \right] - \bar{\multipliers}_1
    \left( \objective\left(\parameters'\right) - \inf_{\parameters' \in
    \Parameters} \objective\left(\parameters'\right) \right) + \left(1 -
    \bar{\multipliers}_1\right) \gamma \\
    \ge & -\bar{\multipliers}_1 \bound{\objective} + \left( 1 -
    \bar{\multipliers}_1 \right) \gamma
  \end{align*}
  Solving for $\bar{\multipliers}_1$ yields the claim.
\end{prf}

%% file: theorems/thm-sparse-equilibrium.tex
\begin{thm}{sparse-equilibrium}
  Consider a two player game, played on the compact Hausdorff spaces
  $\Parameters$ and $\Multipliers \subseteq \R^{\numconstraints}$.
  Imagine that the $\parameters$-player wishes to minimize
  $\lagrangian_{\parameters} : \Parameters \times \Multipliers \rightarrow \R$,
  and the $\multipliers$-player wishes to maximize $\lagrangian_{\multipliers}
  : \Parameters \times \Multipliers \rightarrow \R$, with both of these
  functions being continuous in $\parameters$ and linear in $\multipliers$.
  Then there exists a Nash equilibrium $\parameters$, $\multipliers$:
  \begin{align*}
    \expectation_{\parameters}\left[
    \lagrangian_{\parameters}\left(\parameters, \multipliers\right) \right] =&
    \min_{\parameters^* \in \Parameters}
    \lagrangian_{\parameters}\left(\parameters^*, \multipliers\right) \\
    \expectation_{\parameters}\left[
    \lagrangian_{\multipliers}\left(\parameters, \multipliers\right) \right] =&
    \max_{\multipliers^* \in \Multipliers} \expectation_{\parameters}\left[
    \lagrangian_{\multipliers}\left(\parameters, \multipliers^*\right) \right]
  \end{align*}
  where $\parameters$ is a random variable placing nonzero probability mass on
  \emph{at most} $\numconstraints+1$ elements of $\Parameters$, and
  $\multipliers \in \Multipliers$ is non-random.
\end{thm}
\begin{prf}{sparse-equilibrium}
  There are some extremely similar (and in some ways more general) results than
  this in the game theory
  literature~\egcite{Bohnenblust:1950,Parthasarathy:1975}, but for our
  particular (Lagrangian and proxy-Lagrangian) setting it's possible to provide
  a fairly straightforward proof.

  To begin with, \citet{Glicksburg:1952} gives that there exists a mixed
  strategy in the form of two random variables $\tilde{\parameters}$ and
  $\tilde{\multipliers}$:
  \begin{align*}
    \expectation_{\tilde{\parameters},\tilde{\multipliers}}\left[
    \lagrangian_{\parameters}\left(\tilde{\parameters},
    \tilde{\multipliers}\right) \right] =& \min_{\parameters^* \in \Parameters}
    \expectation_{\tilde{\multipliers}}\left[
    \lagrangian_{\parameters}\left(\parameters^*, \tilde{\multipliers}\right)
    \right] \\
    \expectation_{\tilde{\parameters},\tilde{\multipliers}}\left[
    \lagrangian_{\multipliers}\left(\tilde{\parameters},
    \tilde{\multipliers}\right) \right] =& \max_{\multipliers^* \in
    \Multipliers} \expectation_{\tilde{\parameters}}\left[
    \lagrangian_{\multipliers}\left(\tilde{\parameters}, \multipliers^*\right)
    \right]
  \end{align*}
  Since both functions are linear in $\tilde{\multipliers}$, we can define
  $\multipliers \defeq
  \expectation_{\tilde{\multipliers}}\left[\tilde{\multipliers}\right]$, and
  these conditions become:
  \begin{align*}
    \expectation_{\tilde{\parameters}}\left[
    \lagrangian_{\parameters}\left(\tilde{\parameters}, \multipliers\right)
    \right] =& \min_{\parameters^* \in \Parameters}
    \lagrangian_{\parameters}\left(\parameters^*, \multipliers\right) \defeq
    \ell_{\min} \\
    \expectation_{\tilde{\parameters}}\left[
    \lagrangian_{\multipliers}\left(\tilde{\parameters}, \multipliers\right)
    \right] =& \max_{\multipliers^* \in \Multipliers}
    \expectation_{\tilde{\parameters}}\left[
    \lagrangian_{\multipliers}\left(\tilde{\parameters}, \multipliers^*\right)
    \right]
  \end{align*}
  Let's focus on the first condition. Let $p_{\epsilon} \defeq
  \probability\left\{ \lagrangian_{\parameters}\left(\tilde{\parameters},
  \multipliers\right) \ge \ell_{\min} + \epsilon \right\}$, and notice that
  $p_{1/n}$ must equal zero for any $n \in \left\{1,2,\dots\right\}$ (otherwise
  we would contradict the above), implying by the countable additivity of
  measures that $\probability\left\{
    \lagrangian_{\parameters}\left(\tilde{\parameters}, \multipliers\right) =
    \ell_{\min} \right\} = 1$.
  We therefore assume henceforth, without loss of generality, that the support
  of $\tilde{\parameters}$ consists entirely of minimizers of
  $\lagrangian_{\parameters}\left(\cdot,\multipliers\right)$. Let $S \subseteq
  \Parameters$ be this support set.

  Define $G \defeq \left\{ \grad_{\tilde{\multipliers}}
  \lagrangian_{\multipliers}\left( \parameters', \multipliers \right) :
  \parameters' \in S \right\}$, and take $\bar{G}$ to be the closure of the
  convex hull of $G$.
  Since $\expectation_{\tilde{\parameters}}\left[ \grad_{\tilde{\multipliers}}
  \lagrangian_{\multipliers}\left(\tilde{\parameters},\multipliers\right)
  \right] \in \bar{G} \subseteq \R^{\numconstraints}$, we can write it as
  a convex combination of at most $\numconstraints+1$ extreme points of
  $\bar{G}$, or equivalently of $\numconstraints+1$ elements of $G$.
  Hence, we can take $\parameters$ to be a discrete random variable that places
  nonzero mass on at most $\numconstraints+1$ elements of $S$, and:
  \begin{equation*}
    \expectation_{\parameters}\left[ \grad_{\tilde{\multipliers}}
    \lagrangian_{\multipliers}\left(\parameters,\multipliers\right) \right] =
    \expectation_{\tilde{\parameters}}\left[ \grad_{\tilde{\multipliers}}
    \lagrangian_{\multipliers}\left(\tilde{\parameters},\multipliers\right)
    \right]
  \end{equation*}
  Linearity in $\lambda$ then implies that $\expectation_{\parameters}\left[
  \lagrangian_{\multipliers}\left(\parameters,\cdot\right) \right]$ and
  $\expectation_{\tilde{\parameters}}\left[
  \lagrangian_{\multipliers}\left(\tilde{\parameters},\cdot\right) \right]$ are
  the \emph{same function} (up to a constant), and therefore have the same
  maximizer(s). Correspondingly, $\parameters$ is supported on $S$, which
  contains only minimizers of
  $\lagrangian_{\parameters}\left(\cdot,\multipliers\right)$ by construction.
\end{prf}

%% file: theorems/thm-mirror.tex
\begin{thm}{mirror}
  \textbf{(Mirror Descent)}
  Let $f_1,f_2,\ldots : \Parameters \rightarrow \R$ be a sequence of convex
  functions that we wish to minimize on a compact convex set $\Parameters$.
  Suppose that the ``distance generating function'' $\Psi : \Parameters
  \rightarrow \R_+$ is nonnegative and $1$-strongly convex \wrt a norm
  $\norm{\cdot}$ with dual norm $\norm{\cdot}_*$.

  Define the step size $\eta = \sqrt{ \bound{\Psi} / T \bound{\subgrad}^2 }$,
  where $\bound{\Psi} \ge \max_{\parameters \in \Parameters}
  \Psi\left(\parameters\right)$ is a uniform upper bound on $\Psi$, and
  $\bound{\subgrad} \ge \norm{\subgrad f_t \left(\parameters^{(t)}\right)}_*$
  is a uniform upper bound on the norms of the subgradients.
  Suppose that we perform $T$ iterations of the following update, starting from
  $\parameters^{(1)} = \argmin_{\parameters \in \Parameters}
  \Psi\left(\parameters\right)$:
  \begin{align*}
    \tilde{\parameters}^{(t+1)} =& \grad \Psi^*\left( \grad \Psi\left(
    \parameters^{(t)} \right) - \eta \subgrad f_t\left( \parameters^{(t)}
    \right) \right) \\
    \parameters^{(t+1)} =& \argmin_{\parameters \in \Parameters}
    \bregman{\Psi}\left( \parameters \mid \tilde{\parameters}^{(t+1)} \right)
  \end{align*}
  where $\subgrad f_t\left(\parameters\right) \in \partial
  f_t(\parameters^{(t)})$ is a subgradient of $f_t$ at $\parameters$, and
  $\bregman{\Psi}\left(\parameters \mid \parameters'\right) \defeq
  \Psi\left(\parameters\right) - \Psi\left(\parameters'\right) - \inner{\grad
  \Psi\left(\parameters'\right)}{\parameters - \parameters'}$ is the Bregman
  divergence associated with $\Psi$. Then:
  \begin{equation*}
    \frac{1}{T} \sum_{t=1}^T f_t\left( \parameters^{(t)} \right) - \frac{1}{T}
    \sum_{t=1}^T f_t\left( \parameters^* \right) \le 2 \bound{\subgrad} \sqrt{
    \frac{\bound{\Psi}}{T} }
  \end{equation*}
  where $\parameters^* \in \Parameters$ is an arbitrary reference vector.
\end{thm}
\begin{prf}{mirror}
  Mirror descent~\citep{Nemirovski:1983,Beck:2003} dates back to 1983, but this
  particular statement is taken from Lemma 2 of \citet{Srebro:2011}.
\end{prf}

%% file: theorems/cor-sgd.tex
\begin{cor}{sgd}
  \textbf{(Gradient Descent)}
  Let $f_1,f_2,\ldots : \Parameters \rightarrow \R$ be a sequence of convex
  functions that we wish to minimize on a compact convex set $\Parameters$.

  Define the step size $\eta = \bound{\Parameters} / \bound{\subgrad} \sqrt{2
  T}$, where $\bound{\Parameters} \ge \max_{\parameters \in \Parameters}
  \norm{\parameters}_2$, and $\bound{\subgrad} \ge \norm{\subgrad f_t
  \left(\parameters^{(t)}\right)}_2$ is a uniform upper bound on the norms of
  the subgradients.
  Suppose that we perform $T$ iterations of the following update, starting from
  $\parameters^{(1)} = \argmin_{\parameters \in \Parameters}
  \norm{\parameters}_2$:
  \begin{equation*}
    \parameters^{(t+1)} = \Pi_{\Parameters}\left( \parameters^{(t)} - \eta
    \subgrad f_t\left( \parameters^{(t)} \right) \right)
  \end{equation*}
  where $\subgrad f_t\left(\parameters\right) \in \partial
  f_t(\parameters^{(t)})$ is a subgradient of $f_t$ at $\parameters$, and
  $\Pi_{\Parameters}$ projects its argument onto $\Parameters$ \wrt the
  Euclidean norm. Then:
  \begin{equation*}
    \frac{1}{T} \sum_{t=1}^T f_t\left( \parameters^{(t)} \right) - \frac{1}{T}
    \sum_{t=1}^T f_t\left( \parameters^* \right) \le \bound{\Parameters}
    \bound{\subgrad} \sqrt{\frac{2}{T}}
  \end{equation*}
  where $\parameters^* \in \Parameters$ is an arbitrary reference vector.
\end{cor}
\begin{prf}{sgd}
  Follows from taking $\Psi\left(\parameters\right) = \norm{\parameters}_2^2 /
  2$ in \thmref{mirror}.
\end{prf}

%% file: theorems/cor-matrix-multiplicative.tex
\begin{cor}{matrix-multiplicative}
  Let $\Matrixmultipliers \defeq \left\{ \matrixmultipliers \in
  \R^{\matrixmultipliersize \times \matrixmultipliersize} : \forall i \in
  \indices{\matrixmultipliersize} . \matrixmultipliers_{:, i} \in
  \Delta^{\matrixmultipliersize} \right\}$ be the set of all left-stochastic
  $\matrixmultipliersize \times \matrixmultipliersize$ matrices, and let
  $f_1,f_2,\ldots : \Matrixmultipliers \rightarrow \R$ be a sequence of concave
  functions that we wish to maximize.

  Define the step size $\eta = \sqrt{ \matrixmultipliersize \ln
  \matrixmultipliersize / T \bound{\supgrad}^2 }$, where
  $\bound{\supgrad} \ge \norm{\supgrad f_t\left( \matrixmultipliers^{(t)} \right)}_{\infty, 2}$ is a
  uniform upper bound on the norms of the supergradients, and
  $\norm{\cdot}_{\infty, 2} \defeq \sqrt{ \sum_{i=1}^{\matrixmultipliersize}
  \norm{\matrixmultipliers_{:,i}}_{\infty}^2 }$ is the $L_{\infty,2}$ matrix
  norm.
  Suppose that we perform $T$ iterations of the following update
  starting from the matrix $\matrixmultipliers^{(1)}$ with all elements equal
  to $1/\matrixmultipliersize$:
  \begin{align*}
    \tilde{\matrixmultipliers}^{(t+1)} =& \matrixmultipliers^{(t)}
    \elementwiseproduct \elementwiseexp\left( \eta \supgrad f_t\left(
    \matrixmultipliers^{(t)} \right) \right) \\
    \matrixmultipliers_{:,i}^{(t+1)} =&
    \tilde{\matrixmultipliers}_{:,i}^{(t+1)} /
    \norm{\tilde{\matrixmultipliers}_{:,i}^{(t+1)}}_1
  \end{align*}
  where $-\supgrad f_t\left(\matrixmultipliers^{(t)}\right) \in \partial
  \left(-f_t(\matrixmultipliers^{(t)})\right)$, \ie $\supgrad
  f_t\left(\matrixmultipliers^{(t)}\right)$ is a supergradient of $f_t$ at
  $\matrixmultipliers^{(t)}$, and the multiplication and exponentiation in the
  first step are performed element-wise. Then:
  \begin{equation*}
    \frac{1}{T} \sum_{t=1}^T f_t\left( \matrixmultipliers^* \right) -
    \frac{1}{T} \sum_{t=1}^T f_t\left( \matrixmultipliers^{(t)} \right) \le 2
    \bound{\supgrad} \sqrt{ \frac{ \matrixmultipliersize \ln
    \matrixmultipliersize }{T} }
  \end{equation*}
  where $\matrixmultipliers^* \in \Matrixmultipliers$ is an arbitrary reference
  matrix.
\end{cor}
\begin{prf}{matrix-multiplicative}
  Define $\Psi:\Matrixmultipliers \rightarrow \R \defeq \matrixmultipliersize
  \ln \matrixmultipliersize + \sum_{i,j \in \indices{\matrixmultipliersize}}
  \matrixmultipliers_{i,j} \ln \matrixmultipliers_{i,j}$ as
  $\matrixmultipliersize \ln \matrixmultipliersize$ plus the negative Shannon
  entropy, applied to its (matrix) argument element-wise
  ($\matrixmultipliersize \ln \matrixmultipliersize$ is added to make $\Psi$
  nonnegative on $\Matrixmultipliers$). As in the vector setting, the resulting
  mirror descent update will be (element-wise) multiplicative.

  The Bregman divergence satisfies:
  \begin{align}
    \notag \bregman{\Psi}\left(\matrixmultipliers \vert
    \matrixmultipliers'\right) =& \Psi\left(\matrixmultipliers\right) -
    \Psi\left(\matrixmultipliers'\right) - \inner{\nabla
    \Psi\left(\matrixmultipliers'\right)}{\matrixmultipliers -
    \matrixmultipliers'} \\
    \label{eq:cor:matrix-multiplicative:bregman} =&
    \norm{\matrixmultipliers'}_{1,1} - \norm{\matrixmultipliers}_{1,1} +
    \sum_{i=1}^{\matrixmultipliersize} D_{KL}\left( \matrixmultipliers_{:,i}
    \Vert \matrixmultipliers_{:,i}' \right)
  \end{align}
  where $\norm{\matrixmultipliers}_{1,1} = \sum_{i=1}^{\matrixmultipliersize}
  \norm{\matrixmultipliers_{:,i}}_1$ is the $L_{1,1}$ matrix norm.
  This incidentally shows that one projects onto $\Matrixmultipliers$ \wrt
  $\bregman{\Psi}$ by projecting each column \wrt the KL divergence, \ie by
  normalizing the columns.

  By Pinsker's inequality (applied to each column of an $\matrixmultipliers \in
  \Matrixmultipliers$):
  \begin{equation*}
    \norm{\matrixmultipliers - \matrixmultipliers'}_{1,2}^2 \le 2
    \sum_{i=1}^{\matrixmultipliersize} D_{KL}\left( \matrixmultipliers_{:,i}
    \Vert \matrixmultipliers_{:,i}' \right)
  \end{equation*}
  where $\norm{\matrixmultipliers}_{1,2} =
  \sqrt{\sum_{i=1}^{\matrixmultipliersize}
  \norm{\matrixmultipliers_{:,i}}_1^2}$ is the $L_{1,2}$ matrix norm.
  Substituting this into \eqref{cor:matrix-multiplicative:bregman}, and using
  the fact that $\norm{\matrixmultipliers}_{1,1} = \matrixmultipliersize$ for
  all $\matrixmultipliers \in \Matrixmultipliers$, we have that for all
  $\matrixmultipliers,\matrixmultipliers' \in \Matrixmultipliers$:
  \begin{equation*}
    \bregman{\Psi}\left(\matrixmultipliers \vert \matrixmultipliers'\right) \ge
    \frac{1}{2} \norm{\matrixmultipliers - \matrixmultipliers'}_{1,2}^2
  \end{equation*}
  which shows that $\Psi$ is $1$-strongly convex \wrt the $L_{1,2}$ matrix
  norm.
  The dual norm of the $L_{1,2}$ matrix norm is the $L_{\infty,2}$ norm, which
  is the last piece needed to apply \thmref{mirror}, yielding the claimed
  result.
\end{prf}

%% file: theorems/lem-internal-regret.tex
\begin{lem}{internal-regret}
  Let $\Multipliers \defeq \Delta^{\matrixmultipliersize}$ be the
  $\matrixmultipliersize$-dimensional simplex, define $\Matrixmultipliers
  \defeq \left\{ \matrixmultipliers \in \R^{\matrixmultipliersize \times
  \matrixmultipliersize} : \forall i \in \indices{\matrixmultipliersize} .
  \matrixmultipliers_{:, i} \in \Delta^{\matrixmultipliersize} \right\}$ as the
  set of all left-stochastic $\matrixmultipliersize \times
  \matrixmultipliersize$ matrices, and take $f_1,f_2,\ldots : \Multipliers
  \rightarrow \R$ to be a sequence of concave functions that we wish to
  maximize.

  Define the step size $\eta = \sqrt{ \matrixmultipliersize \ln
  \matrixmultipliersize / T \bound{\supgrad}^2 }$, where
  $\bound{\supgrad} \ge \norm{\supgrad f_t\left(\multipliers^{(t)}\right) }_{\infty}$ is a
  uniform upper bound on the $\infty$-norms of the supergradients.
  Suppose that we perform $T$ iterations of the following update, starting from
  the matrix $\matrixmultipliers^{(1)}$ with all elements equal to
  $1/\matrixmultipliersize$:
  \begin{align*}
    \multipliers^{(t)} =& \fix \matrixmultipliers^{(t)} \\
    \deltamatrix^{(t)} =& \left( \supgrad f_t\left(\multipliers^{(t)}\right)
    \right) \left( \multipliers^{(t)} \right) ^T \\
    \tilde{\matrixmultipliers}^{(t+1)} =& \matrixmultipliers^{(t)}
    \elementwiseproduct \elementwiseexp\left( \eta \deltamatrix^{(t)} \right)
    \\
    \matrixmultipliers_{:,i}^{(t+1)} =&
    \tilde{\matrixmultipliers}_{:,i}^{(t+1)} /
    \norm{\tilde{\matrixmultipliers}_{:,i}^{(t+1)}}_1
  \end{align*}
  where $\fix \matrixmultipliers$ is a stationary distribution of
  $\matrixmultipliers$ (\ie a $\multipliers \in \Multipliers$ such that
  $\matrixmultipliers \multipliers = \multipliers$---such always exists, since
  $\matrixmultipliers$ is left-stochastic), $-\supgrad
  f_t\left(\multipliers^{(t)}\right) \in \partial
  \left(-f_t(\multipliers^{(t)})\right)$, \ie $\supgrad
  f_t\left(\multipliers^{(t)}\right)$ is a supergradient of $f_t$ at
  $\multipliers^{(t)}$, and the multiplication and exponentiation of the third
  step are performed element-wise. Then:
  \begin{equation*}
    \frac{1}{T} \sum_{t=1}^T f_t\left( \matrixmultipliers^* \multipliers^{(t)}
    \right) - \frac{1}{T} \sum_{t=1}^T f_t\left( \multipliers^{(t)} \right) \le
    2 \bound{\supgrad} \sqrt{ \frac{ \matrixmultipliersize \ln
    \matrixmultipliersize }{T} }
  \end{equation*}
  where $\matrixmultipliers^* \in \Matrixmultipliers$ is an arbitrary
  left-stochastic reference matrix.
\end{lem}
\begin{prf}{internal-regret}
  This algorithm is an instance of that contained in Figure 1 of
  \citet{Gordon:2008}.

  Define $\tilde{f}_t\left(\matrixmultipliers\right) \defeq
  f_t\left(\matrixmultipliers^{(t)} \multipliers^{(t)}\right)$. Observe that
  since $\supgrad f_t\left(\multipliers^{(t)}\right)$ is a supergradient of
  $f_t$ at $\multipliers^{(t)}$, and $\matrixmultipliers^{(t)}
  \multipliers^{(t)} = \multipliers^{(t)}$:
  \begin{align*}
    f_t\left(\tilde{\matrixmultipliers} \multipliers^{(t)}\right) \le &
    f_t\left(\matrixmultipliers^{(t)} \multipliers^{(t)}\right) +
    \inner{\supgrad
    f_t\left(\multipliers^{(t)}\right)}{\tilde{\matrixmultipliers}
    \multipliers^{(t)} - \matrixmultipliers^{(t)} \multipliers^{(t)}} \\
    \le & f_t\left(\matrixmultipliers^{(t)} \multipliers^{(t)}\right) +
    \deltamatrix^{(t)} \cdot \left( \tilde{\matrixmultipliers} -
    \matrixmultipliers^{(t)} \right)
  \end{align*}
  where the matrix product on the last line is performed element-wise.  This
  shows that $\deltamatrix^{(t)}$ is a supergradient of $\tilde{f}_t$ at
  $\matrixmultipliers^{(t)}$, from which we conclude that the final two steps
  of the update are performing the algorithm of \corref{matrix-multiplicative},
  so:
  \begin{equation*}
    \frac{1}{T} \sum_{t=1}^T \tilde{f}_t\left( \matrixmultipliers^* \right) -
    \frac{1}{T} \sum_{t=1}^T \tilde{f}_t\left( \matrixmultipliers^{(t)} \right)
    \le 2 \bound{\supgrad} \sqrt{ \frac{ \matrixmultipliersize \ln
    \matrixmultipliersize }{T} }
  \end{equation*}
  where the $\bound{\supgrad}$ of \corref{matrix-multiplicative} is a uniform
  upper bound on the $L_{\infty,2}$ matrix norms of the $\deltamatrix^{(t)}$s.
  However, by the definition of $\deltamatrix^{(t)}$ and the fact that
  $\multipliers^{(t)} \in \Delta^{\matrixmultipliersize}$, we can instead take
  $\bound{\supgrad}$ to be a uniform upper bound on
  $\norm{\supgrad^{(t)}}_\infty$.
  Substituting the definition of $\tilde{f}_t$ and again using the fact that
  $\matrixmultipliers^{(t)} \multipliers^{(t)} = \multipliers^{(t)}$ then
  yields the claimed result.
\end{prf}

%% file: theorems/thm-stochastic-mirror.tex
\begin{thm}{stochastic-mirror}
  \textbf{(Stochastic Mirror Descent)}
  Let $\Psi$, $\norm{\cdot}$, $\bregman{\Psi}$ and $\bound{\Psi}$ be as in
  \thmref{mirror}, and let $f_1,f_2,\ldots : \Parameters \rightarrow \R$ be a
  sequence of convex functions that we wish to minimize on a compact convex set
  $\Parameters$.

  Define the step size $\eta = \sqrt{ \bound{\Psi} / T
  \bound{\stochasticsubgrad}^2 }$, where $\bound{\stochasticsubgrad} \ge
  \norm{\stochasticsubgrad^{(t)}}_*$ is a uniform upper bound on the norms of the
  stochastic subgradients.
  Suppose that we perform $T$ iterations of the following \emph{stochastic}
  update, starting from $\parameters^{(1)} = \argmin_{\parameters \in
  \Parameters} \Psi\left(\parameters\right)$:
  \begin{align*}
    \tilde{\parameters}^{(t+1)} &= \grad \Psi^* \left( \grad \Psi \left(
    \parameters^{(t)} \right) - \eta \stochasticsubgrad^{(t)} \right) \\
    \parameters^{(t+1)} &= \argmin_{\parameters \in \Parameters}
    \bregman{\Psi}\left(\parameters \vert \tilde{\parameters}^{(t+1)}\right)
  \end{align*}
  where $\expectation\left[ \stochasticsubgrad^{(t)} \mid \parameters^{(t)} \right]
  \in \partial f_t(\parameters^{(t)})$, \ie $\stochasticsubgrad^{(t)}$ is a
  stochastic subgradient of $f_t$ at $\parameters^{(t)}$. Then, with
  probability $1-\delta$ over the draws of the stochastic subgradients:
  \begin{equation*}
    \frac{1}{T} \sum_{t=1}^T f_t\left( \parameters^{(t)} \right) - \frac{1}{T}
    \sum_{t=1}^T f_t\left( \parameters^* \right) \le 2 \bound{\subgrad} \sqrt{ \frac{2
    \bound{\Psi} \left( 1 + 16 \ln\frac{1}{\delta}
    \right)}{T} }
  \end{equation*}
  where $\parameters^* \in \Parameters$ is an arbitrary reference vector.
\end{thm}
\begin{prf}{stochastic-mirror}
  This is nothing more than the usual transformation of a uniform regret
  guarantee into a stochastic one via the Hoeffding-Azuma inequality---we
  include a proof for completeness.

  Define the sequence:
  \begin{equation*}
    \tilde{f}_t\left(\parameters\right) = f_t\left(\parameters^{(t)}\right) +
    \inner{\stochasticsubgrad^{(t)}}{\parameters - \parameters^{(t)}}
  \end{equation*}
  Then applying non-stochastic mirror descent to the sequence $\tilde{f}_t$
  will result in exactly the same sequence of iterates $\parameters^{(t)}$ as
  applying stochastic mirror descent (above) to $f_t$. Hence, by
  \thmref{mirror} and the definition of $\tilde{f}_t$ (notice that we can take
  $\bound{\subgrad} = \bound{\stochasticsubgrad}$):
  \begin{align}
    \notag \frac{1}{T} \sum_{t=1}^T \tilde{f}_t\left( \parameters^{(t)} \right)
    - \frac{1}{T} \sum_{t=1}^T \tilde{f}_t\left( \parameters^* \right) \le& 2 \bound{\subgrad}
    \sqrt{ \frac{\bound{\Psi}}{T} } \\
    \notag \frac{1}{T} \sum_{t=1}^T f_t\left( \parameters^{(t)} \right) -
    \frac{1}{T} \sum_{t=1}^T f_t\left( \parameters^* \right) \le& 2 \bound{\subgrad} \sqrt{
    \frac{\bound{\Psi}}{T} } + \frac{1}{T} \sum_{t=1}^T
    \left( \tilde{f}_t\left(\parameters^*\right) -
    f_t\left(\parameters^*\right) \right) \\
    \label{eq:thm:stochastic-mirror:before-azuma} \le & 2 \bound{\subgrad} \sqrt{
    \frac{\bound{\Psi}}{T} } + \frac{1}{T} \sum_{t=1}^T
    \inner{\stochasticsubgrad^{(t)} - \subgrad
    f_t\left(\parameters^{(t)}\right)}{\parameters^* - \parameters^{(t)}}
  \end{align}
  where the last step follows from the convexity of the $f_t$s. Consider the
  second term on the RHS. Observe that, since the $\stochasticsubgrad^{(t)}$s are
  stochastic subgradients, each of the terms in the sum is zero in expectation
  (conditioned on the past), and the partial sums therefore form a martingale.
  Furthermore, by H\"older's inequality:
  \begin{equation*}
    \inner{\stochasticsubgrad^{(t)} - \subgrad
    f_t\left(\parameters^{(t)}\right)}{\parameters^* - \parameters^{(t)}}
    \le \norm{\stochasticsubgrad^{(t)} - \subgrad
    f_t\left(\parameters^{(t)}\right)}_* \norm{\parameters^* -
    \parameters^{(t)}}
    \le 4 \bound{\stochasticsubgrad} \sqrt{2 \bound{\Psi}}
  \end{equation*}
  the last step because $\norm{\parameters^* - \parameters^{(t)}} \le
  \norm{\parameters^* - \parameters^{(1)}} + \norm{\parameters^{(t)} -
  \parameters^{(1)}} \le 2 \sup_{\parameters \in \Parameters} \sqrt{2
  \bregman{\Psi}\left(\parameters \mid \parameters^{(1)}\right)} \le 2 \sqrt{2
  \bound{\Psi}}$, using the fact that $\bregman{\Psi}$ is $1$-strongly convex
  \wrt $\norm{\cdot}$, and the definition of $\parameters^{(1)}$.
  Hence, by the Hoeffding-Azuma inequality:
  \begin{equation*}
    \probability\left\{ \frac{1}{T} \sum_{t=1}^T \inner{\stochasticsubgrad^{(t)} -
    \subgrad f_t\left(\parameters^{(t)}\right)}{\parameters^* -
    \parameters^{(t)}} \ge \epsilon \right\}
    \le \exp\left( -\frac{T \epsilon^2}{64 \bound{\Psi}
    \bound{\stochasticsubgrad}^2} \right)
  \end{equation*}
  equivalently:
  \begin{equation*}
    \probability\left\{ \frac{1}{T} \sum_{t=1}^T \inner{\stochasticsubgrad^{(t)} -
    \subgrad f_t\left(\parameters^{(t)}\right)}{\parameters^* -
    \parameters^{(t)}} \ge 8 \bound{\stochasticsubgrad} \sqrt{\frac{ \bound{\Psi}
    \ln \frac{1}{\delta} }{T}} \right\}
    \le \delta
  \end{equation*}
  substituting this into \eqref{thm:stochastic-mirror:before-azuma}, and
  applying the inequality $\sqrt{a} + \sqrt{b} \le \sqrt{2a + 2b}$, yields the
  claimed result.
\end{prf}

%% file: theorems/cor-stochastic-sgd.tex
\begin{cor}{stochastic-sgd}
  \textbf{(Stochastic Gradient Descent)}
  Let $f_1,f_2,\ldots : \Parameters \rightarrow \R$ be a sequence of convex
  functions that we wish to minimize on a compact convex set $\Parameters$.

  Define the step size $\eta = \bound{\Parameters} / \bound{\stochasticsubgrad}
  \sqrt{2 T}$, where $\bound{\Parameters} \ge \max_{\parameters \in \Parameters}
  \norm{\parameters}_2$, and $\bound{\stochasticsubgrad} \ge
  \norm{\stochasticsubgrad^{(t)}}_2$ is a uniform upper bound on the norms of the
  stochastic subgradients.
  Suppose that we perform $T$ iterations of the following \emph{stochastic}
  update, starting from $\parameters^{(1)} = \argmin_{\parameters \in
  \Parameters} \norm{\parameters}_2$:
  \begin{equation*}
    \parameters^{(t+1)} = \Pi_{\Parameters}\left( \parameters^{(t)} - \eta
    \stochasticsubgrad^{(t)} \right)
  \end{equation*}
  where $\expectation\left[ \stochasticsubgrad^{(t)} \mid \parameters^{(t)} \right]
  \in \partial f_t(\parameters^{(t)})$, \ie $\stochasticsubgrad^{(t)}$ is a
  stochastic subgradient of $f_t$ at $\parameters^{(t)}$, and
  $\Pi_{\Parameters}$ projects its argument onto $\Parameters$ \wrt the
  Euclidean norm.  Then, with probability $1-\delta$ over the draws of the
  stochastic subgradients:
  \begin{equation*}
    \frac{1}{T} \sum_{t=1}^T f_t\left( \parameters^{(t)} \right) - \frac{1}{T}
    \sum_{t=1}^T f_t\left( \parameters^* \right) \le 2 \bound{\Parameters}
    \bound{\subgrad} \sqrt{ \frac{ 1 + 16 \ln\frac{1}{\delta} }{T} }
  \end{equation*}
  where $\parameters^* \in \Parameters$ is an arbitrary reference vector.
\end{cor}
\begin{prf}{stochastic-sgd}
  Follows from taking $\Psi\left(\parameters\right) = \norm{\parameters}_2^2 /
  2$ in \thmref{stochastic-mirror}.
\end{prf}

%% file: theorems/cor-stochastic-matrix-multiplicative.tex
\begin{cor}{stochastic-matrix-multiplicative}
  Let $\Matrixmultipliers \defeq \left\{ \matrixmultipliers \in
  \R^{\matrixmultipliersize \times \matrixmultipliersize} : \forall i \in
  \indices{\matrixmultipliersize} . \matrixmultipliers_{:, i} \in
  \Delta^{\matrixmultipliersize} \right\}$ be the set of all left-stochastic
  $\matrixmultipliersize \times \matrixmultipliersize$ matrices, and let
  $f_1,f_2,\ldots : \Matrixmultipliers \rightarrow \R$ be a sequence of concave
  functions that we wish to maximize.

  Define the step size $\eta = \sqrt{ \matrixmultipliersize \ln
  \matrixmultipliersize / T \bound{\stochasticsupgrad}^2 }$, where
  $\bound{\stochasticsupgrad} \ge \norm{\stochasticsupgrad^{(t)}}_{\infty, 2}$ is a
  uniform upper bound on the norms of the stochastic supergradients, and
  $\norm{\cdot}_{\infty, 2} \defeq \sqrt{ \sum_{i=1}^{\matrixmultipliersize}
  \norm{\matrixmultipliers_{:,i}}_{\infty}^2 }$ is the $L_{\infty,2}$ matrix
  norm.
  Suppose that we perform $T$ iterations of the following stochastic update
  starting from the matrix $\matrixmultipliers^{(1)}$ with all elements equal
  to $1/\matrixmultipliersize$:
  \begin{align*}
    \tilde{\matrixmultipliers}^{(t+1)} =& \matrixmultipliers^{(t)}
    \elementwiseproduct \elementwiseexp\left( \eta \stochasticsupgrad^{(t)}
    \right) \\
    \matrixmultipliers_{:,i}^{(t+1)} =&
    \tilde{\matrixmultipliers}_{:,i}^{(t+1)} /
    \norm{\tilde{\matrixmultipliers}_{:,i}^{(t+1)}}_1
  \end{align*}
  where $\expectation\left[ - \stochasticsupgrad^{(t)} \mid
  \matrixmultipliers^{(t)}\right] \in \partial
  \left(-f_t(\matrixmultipliers^{(t)})\right)$, \ie $\stochasticsupgrad^{(t)}$
  is a stochastic supergradient of $f_t$ at $\matrixmultipliers^{(t)}$, and the
  multiplication and exponentiation in the first step are performed
  element-wise. Then with probability $1-\delta$ over the draws of the
  stochastic supergradients:
  \begin{equation*}
    \frac{1}{T} \sum_{t=1}^T f_t\left( \matrixmultipliers^* \right) -
    \frac{1}{T} \sum_{t=1}^T f_t\left( \matrixmultipliers^{(t)} \right) \le 2
    \bound{\stochasticsupgrad} \sqrt{ \frac{ 2 \left( \matrixmultipliersize \ln
    \matrixmultipliersize \right) \left( 1 + 16 \ln\frac{1}{\delta}\right) }{T}
    }
  \end{equation*}
  where $\matrixmultipliers^* \in \Matrixmultipliers$ is an arbitrary reference
  matrix.
\end{cor}
\begin{prf}{stochastic-matrix-multiplicative}
  The same reasoning as was used to prove \corref{matrix-multiplicative} from
  \thmref{mirror} applies here (but starting from \thmref{stochastic-mirror}).
\end{prf}

%% file: theorems/lem-stochastic-internal-regret.tex
\begin{lem}{stochastic-internal-regret}
  Let $\Multipliers \defeq \Delta^{\matrixmultipliersize}$ be the
  $\matrixmultipliersize$-dimensional simplex, define $\Matrixmultipliers
  \defeq \left\{ \matrixmultipliers \in \R^{\matrixmultipliersize \times
  \matrixmultipliersize} : \forall i \in \indices{\matrixmultipliersize} .
  \matrixmultipliers_{:, i} \in \Delta^{\matrixmultipliersize} \right\}$ as the
  set of all left-stochastic $\matrixmultipliersize \times
  \matrixmultipliersize$ matrices, and take $f_1,f_2,\ldots : \Multipliers
  \rightarrow \R$ to be a sequence of concave functions that we wish to
  maximize.

  Define the step size $\eta = \sqrt{ \matrixmultipliersize \ln
  \matrixmultipliersize / T \bound{\stochasticsupgrad}^2 }$, where
  $\bound{\stochasticsupgrad} \ge \norm{\stochasticsupgrad^{(t)}}_{\infty}$ is a
  uniform upper bound on the $\infty$-norms of the stochastic supergradients.
  Suppose that we perform $T$ iterations of the following update, starting from
  the matrix $\matrixmultipliers^{(1)}$ with all elements equal to
  $1/\matrixmultipliersize$:
  \begin{align*}
    \multipliers^{(t)} =& \fix \matrixmultipliers^{(t)} \\
    \deltamatrix^{(t)} =& \stochasticsupgrad^{(t)} \left(
    \multipliers^{(t)} \right) ^T \\
    \tilde{\matrixmultipliers}^{(t+1)} =& \matrixmultipliers^{(t)}
    \elementwiseproduct \elementwiseexp\left( \eta \deltamatrix^{(t)} \right)
    \\
    \matrixmultipliers_{:,i}^{(t+1)} =&
    \tilde{\matrixmultipliers}_{:,i}^{(t+1)} /
    \norm{\tilde{\matrixmultipliers}_{:,i}^{(t+1)}}_1
  \end{align*}
  where $\fix \matrixmultipliers$ is a stationary distribution of
  $\matrixmultipliers$ (\ie a $\multipliers \in \Multipliers$ such that
  $\matrixmultipliers \multipliers = \multipliers$---such always exists, since
  $\matrixmultipliers$ is left-stochastic), $\expectation\left[ -
  \stochasticsupgrad^{(t)} \mid \multipliers^{(t)} \right] \in \partial
  \left(-f_t(\multipliers^{(t)})\right)$, \ie $\stochasticsupgrad^{(t)}$ is a
  stochastic supergradient of $f_t$ at $\multipliers^{(t)}$, and the
  multiplication and exponentiation of the third step are performed
  element-wise. Then with probability $1-\delta$ over the draws of the
  stochastic supergradients:
  \begin{equation*}
    \frac{1}{T} \sum_{t=1}^T f_t\left( \matrixmultipliers^* \multipliers^{(t)}
    \right) - \frac{1}{T} \sum_{t=1}^T f_t\left( \multipliers^{(t)} \right) \le
    2 \bound{\stochasticsupgrad} \sqrt{ \frac{ 2 \left( \matrixmultipliersize
    \ln \matrixmultipliersize \right) \left( 1 + 16 \ln\frac{1}{\delta}\right)
    }{T} }
  \end{equation*}
  where $\matrixmultipliers^* \in \Matrixmultipliers$ is an arbitrary
  left-stochastic reference matrix.
\end{lem}
\begin{prf}{stochastic-internal-regret}
  The same reasoning as was used to prove \lemref{internal-regret} from
  \corref{matrix-multiplicative} applies here (but starting from
  \corref{stochastic-matrix-multiplicative}).
\end{prf}

%% file: figures/alg-stochastic-lagrangian.tex
\begin{algorithm*}[t]

\begin{pseudocode}
\codename $\mbox{StochasticLagrangian}\left( \Radius \in \R_+, \lagrangian : \Parameters \times \Multipliers \rightarrow \R, T \in \N, \eta_{\parameters}, \eta_{\multipliers} \in \R_+ \right)$: \\
\codeline Initialize $\parameters^{(1)} = 0$, $\multipliers^{(1)} = 0$ \codecomment{Assumes $0 \in \Parameters$} \\
\codeline For $t \in \indices{T}$: \\
\codeline \> Let $\stochasticsubgrad^{(t)}_{\parameters}$ be a stochastic subgradient of $\lagrangian\left(\parameters^{(t)},\multipliers^{(t)}\right)$ \wrt $\parameters$ \\
\codeline \> Let $\stochasticgrad^{(t)}_{\multipliers}$ be a stochastic gradient of $\lagrangian\left(\parameters^{(t)},\multipliers^{(t)}\right)$ \wrt $\multipliers$ \\
\codeline \> Update $\parameters^{(t+1)} = \Pi_{\Parameters}\left( \parameters^{(t)} - \eta_{\parameters} \stochasticsubgrad^{(t)}_{\parameters} \right)$ \codecomment{Projected SGD updates \dots} \\
\codeline \> Update $\multipliers^{(t+1)} = \Pi_{\Multipliers}\left( \multipliers^{(t)} + \eta_{\multipliers} \stochasticgrad^{(t)}_{\multipliers} \right)$ \codecomment{\;\;\;\;\dots} \\
\codeline Return $\parameters^{(1)},\dots,\parameters^{(T)}$ and $\multipliers^{(1)},\dots,\multipliers^{(T)}$
\end{pseudocode}

\caption{
  Optimizes the Lagrangian formulation (\defref{lagrangian}) in the convex
  setting.
  The parameter $\Radius$ is the radius of the Lagrange multiplier space
  $\Multipliers \defeq \left\{ \multipliers \in \R_+^{\numconstraints} :
  \norm{\multipliers}_1 \le \Radius \right\}$, and the functions
  $\Pi_{\Parameters}$ and $\Pi_{\Multipliers}$ project their arguments onto
  $\Parameters$ and $\Multipliers$ (respectively) \wrt the Euclidean norm.
}

\label{alg:stochastic-lagrangian}

\end{algorithm*}

%% file: theorems/lem-stochastic-lagrangian.tex
\begin{lem}{stochastic-lagrangian}
  \ifshowproofs
  \textbf{(\algref{stochastic-lagrangian})}
  \fi
  Suppose that $\Parameters$ is a compact convex set, $\Multipliers$ and
  $\Radius$ are as in \thmref{lagrangian-suboptimality-and-feasibility}, and
  that the objective and constraint functions
  $\objective,\constraint{1},\dots,\constraint{\numconstraints}$ are convex.
  Define the three upper bounds $\bound{\Parameters} \ge \max_{\parameters \in
  \Parameters} \norm{\parameters}_2$, $\bound{\stochasticsubgrad} \ge \max_{t
  \in \indices{T}} \norm{\stochasticsubgrad_{\parameters}^{(t)}}_2$, and
  $\bound{\stochasticgrad} \ge \max_{t \in \indices{T}}
  \norm{\stochasticgrad_{\multipliers}^{(t)}}_2$.

  If we run \algref{stochastic-lagrangian} with the step sizes
  $\eta_{\parameters} \defeq \bound{\Parameters} / \bound{\stochasticsubgrad}
  \sqrt{2T}$ and $\eta_{\multipliers} \defeq \Radius /
  \bound{\stochasticgrad} \sqrt{2T}$, then the result satisfies the
  conditions of \thmref{lagrangian-suboptimality-and-feasibility} for:
  \begin{equation*}
    \epsilon = 2 \left( \bound{\Parameters} \bound{\stochasticsubgrad} +
    \Radius \bound{\stochasticgrad} \right) \sqrt{ \frac{ 1 + 16
    \ln\frac{2}{\delta} }{T} }
  \end{equation*}
  with probability $1-\delta$ over the draws of the stochastic
  (sub)gradients.
\end{lem}
\begin{prf}{stochastic-lagrangian}
  Applying \corref{stochastic-sgd} to the two optimizations (over $\parameters$
  and $\multipliers$) gives that with probability $1-2\delta'$ over the draws
  of the stochastic (sub)gradients:
  \begin{align*}
    \frac{1}{T} \sum_{t=1}^T \lagrangian\left( \parameters^{(t)},
    \multipliers^{(t)} \right) - \frac{1}{T} \sum_{t=1}^T \lagrangian\left(
    \parameters^*, \multipliers^{(t)} \right) \le& 2 \bound{\Parameters}
    \bound{\stochasticsubgrad} \sqrt{ \frac{ 1 + 16 \ln\frac{1}{\delta'} }{T} }
    \\
    \frac{1}{T} \sum_{t=1}^T \lagrangian\left( \parameters^{(t)},
    \multipliers^* \right) - \frac{1}{T} \sum_{t=1}^T \lagrangian\left(
    \parameters^{(t)}, \multipliers^{(t)} \right) \le& 2 \bound{\Multipliers}
    \bound{\stochasticgrad} \sqrt{ \frac{ 1 + 16 \ln\frac{1}{\delta'} }{T} }
  \end{align*}
  Adding these inequalities, taking $\delta=2\delta'$, using the linearity of
  $\lagrangian$ in $\multipliers$, the fact that $\bound{\Multipliers} =
  \Radius$, and the definitions of $\bar{\parameters}$ and
  $\bar{\multipliers}$, yields the claimed result.
\end{prf}

%% file: figures/alg-oracle-proxy-lagrangian.tex
\begin{algorithm*}[t]

\begin{pseudocode}
\codename $\mbox{OracleProxyLagrangian}\left( \lagrangian_{\parameters}, \lagrangian_{\multipliers} : \Parameters \times \Delta^{\numconstraints+1} \rightarrow \R, \oracle : \left(\Parameters \rightarrow \R\right) \rightarrow \Parameters, T \in \N, \eta_{\multipliers} \in \R_+ \right)$: \\
\codeline Initialize $\matrixmultipliers^{(1)} \in \R^{\left(\numconstraints + 1\right) \times \left(\numconstraints + 1\right)}$ with $\matrixmultipliers_{i,j} = 1 / \left(\numconstraints+1\right)$ \\
\codeline For $t \in \indices{T}$: \\
\codeline \> Let $\multipliers^{(t)} = \fix \matrixmultipliers^{(t)}$ \codecomment{Stationary distribution of $\matrixmultipliers^{(t)}$} \\
\codeline \> Let $\parameters^{(t)} = \oracle\left( \lagrangian_{\parameters}\left(\cdot,\multipliers^{(t)}\right) \right)$ \codecomment{Oracle optimization} \\
\codeline \> Let $\stochasticgrad^{(t)}_{\multipliers}$ be a gradient of $\lagrangian_{\multipliers}\left(\parameters^{(t)},\multipliers^{(t)}\right)$ \wrt $\multipliers$ \\
\codeline \> Update $\tilde{\matrixmultipliers}^{(t+1)} = \matrixmultipliers^{(t)} \elementwiseproduct \elementwiseexp\left( \eta_{\multipliers} \stochasticgrad^{(t)}_{\multipliers} \left( \multipliers^{(t)} \right)^T \right)$ \codecomment{$\elementwiseproduct$ and $\elementwiseexp$ are element-wise} \\
\codeline \> Project $\matrixmultipliers^{(t+1)}_{:,i} = \tilde{\matrixmultipliers}^{(t+1)}_{:,i} / \norm{\tilde{\matrixmultipliers}^{(t+1)}_{:,i}}_1$ for $i\in\indices{\numconstraints+1}$ \codecomment{Column-wise projection \wrt KL divergence}  \\
\codeline Return $\parameters^{(1)},\dots,\parameters^{(T)}$ and $\multipliers^{(1)},\dots,\multipliers^{(T)}$
\end{pseudocode}

\caption{
  Optimizes the proxy-Lagrangian formulation (\defref{proxy-lagrangians}) in
  the non-convex setting via the use of an approximate Bayesian optimization
  oracle $\oracle$ (\defref{oracle}, but with $\proxyconstraint{i}$s instead of
  $\constraint{i}$s in the linear combination defining $f$) for the
  $\parameters$-player, with the $\multipliers$-player minimizing swap regret.
  The $\fix \matrixmultipliers$ operation on line $3$ results in a stationary
  distribution of $\matrixmultipliers$ (\ie a $\multipliers \in \Multipliers$
  such that $\matrixmultipliers \multipliers = \multipliers$, which can be
  derived from the top eigenvector).
}

\label{alg:oracle-proxy-lagrangian}

\end{algorithm*}

%% file: theorems/lem-oracle-proxy-lagrangian.tex
\begin{lem}{oracle-proxy-lagrangian}
  \ifshowproofs
  \textbf{(\algref{oracle-proxy-lagrangian})}
  \fi
  Suppose that $\Matrixmultipliers$ and $\Multipliers$ are as in
  \thmref{proxy-lagrangian-suboptimality-and-feasibility}, and define the upper
  bound $\bound{\stochasticgrad} \ge \max_{t \in \indices{T}}
  \norm{\stochasticgrad_{\multipliers}^{(t)}}_{\infty}$.

  If we run \algref{oracle-proxy-lagrangian} with the step size
  $\eta_{\multipliers} \defeq \sqrt{ \left(\numconstraints+1\right) \ln
  \left(\numconstraints+1\right) / T \bound{\stochasticgrad}^2 }$, then the
  result satisfies satisfies the conditions of
  \thmref{proxy-lagrangian-suboptimality-and-feasibility} for:
  \begin{align*}
    \epsilon_{\parameters} =& \approximation \\
    \epsilon_{\multipliers} =& 2 \bound{\stochasticgrad} \sqrt{ \frac{
    \left(\numconstraints+1\right) \ln \left(\numconstraints+1\right) }{T} }
  \end{align*}
  where $\approximation$ is the error associated with the oracle $\oracle$.
\end{lem}
\begin{prf}{oracle-proxy-lagrangian}
  Applying \lemref{internal-regret} to the optimization over $\multipliers$
  (with $\matrixmultipliersize \defeq \numconstraints + 1$) gives:
  \begin{equation*}
    \frac{1}{T} \sum_{t=1}^T \lagrangian_{\multipliers}\left(
    \parameters^{(t)}, \matrixmultipliers^* \multipliers^{(t)} \right) -
    \frac{1}{T} \sum_{t=1}^T \lagrangian_{\multipliers}\left(
    \parameters^{(t)}, \multipliers^{(t)} \right) \le 2
    \bound{\stochasticgrad} \sqrt{ \frac{ \left(\numconstraints+1\right) \ln
    \left(\numconstraints+1\right) }{T} }
  \end{equation*}
  By the definition of $\oracle$ (\defref{oracle}):
  \begin{equation*}
    \frac{1}{T} \sum_{t=1}^T \lagrangian_{\parameters}\left( \parameters^{(t)},
    \multipliers^{(t)} \right) - \inf_{\parameters^* \in \Parameters}
    \frac{1}{T} \sum_{t=1}^T \lagrangian_{\parameters}\left( \parameters^*,
    \multipliers^{(t)} \right) \le \approximation
  \end{equation*}
  Using the definitions of $\bar{\parameters}$ and $\bar{\multipliers}$ yields
  the claimed result.
\end{prf}